\documentclass[11pt]{article}
\pdfoutput=1

\usepackage{fullpage,times}

\usepackage[utf8]{inputenc}
\usepackage[T1]{fontenc}
\usepackage{hyperref}
\hypersetup{
        colorlinks   = true,
        urlcolor     = blue,
        linkcolor    = blue,
        citecolor   = blue
}

\usepackage{amsthm} 
\usepackage{amsmath}
\usepackage{amsfonts}
\usepackage{mathtools} 
\usepackage{url}
\usepackage{booktabs}
\usepackage{amsfonts,amssymb}
\usepackage{nicefrac}
\usepackage{graphics}
\usepackage{authblk}
\usepackage{microtype}
\usepackage{caption}
\usepackage{subcaption}
\usepackage{algorithm}
\usepackage{algorithmic}
\usepackage[round]{natbib}

\newtheorem{lemma}{Lemma}
\newtheorem{assumption}{Assumption}
\newtheorem{theorem}{Theorem}
\newtheorem{corollary}{Corollary}

\newcommand{\norm}[1]{\left\lVert#1\right\rVert}

\newcommand{\btilde}{\widetilde{\mathbf{b}}}

\newcommand{\h}{\mathbf{h}}
\newcommand{\Hbar}{\mathbf{\overline{h}}}
\newcommand{\hhat}{\mathbf{\widehat{h}}}
\newcommand{\Htilde}{\widetilde{H}}
\newcommand{\I}{\mathbf{I}}
\newcommand{\mar}{\mathrm{Mar}}
\newcommand{\bigO}{\mathcal{O}}

\newcommand{\reals}{\mathbb{R}}

\newcommand{\V}{\mathbf{V}}
\newcommand{\var}{\text{Var}}
\newcommand{\Var}{\mathrm{Var}}

\newcommand{\Vtilde}{\widetilde{\mathbf{V}}}
\newcommand{\w}{\mathbf{w}}
\newcommand{\wbar}{\mathbf{\overline{w}}}
\newcommand{\what}{\mathbf{\widehat{w}}}
\newcommand{\wtilde}{\mathbf{\widetilde{w}}}
\newcommand{\x}{\mathbf{x}}
\newcommand{\xtop}{\mathbf{x}^\top}
\newcommand{\xh}{\mathbf{x}^\mathbf{h}}
\newcommand{\xhtop}{\mathbf{x}^{\mathbf{h}\top}}

\allowdisplaybreaks

\title{Meta-learning with Stochastic Linear Bandits}

\author[1,3]{Leonardo Cella \thanks{leonardocella@gmail.com}}
\author[2]{Alessandro Lazaric \thanks{lazaric@fb.com}}
\author[3]{Massimiliano Pontil \thanks{massimiliano.pontil@iit.it}}

\affil[1]{University of Milan}
\affil[2]{Facebook AI Research}
\affil[3]{Istituto Italiano di Tecnologia and University College London}

\makeindex

\begin{document}
\maketitle

\begin{abstract}
    We investigate meta-learning procedures in the setting of stochastic linear bandits tasks. 
    The goal is to select a learning algorithm which works well on average over a class of bandits tasks, that are sampled from a task-distribution. Inspired by recent work on learning-to-learn linear regression, we consider a class of bandit algorithms that implement a regularized version of the well-known OFUL algorithm, where the regularization is a square euclidean distance to a bias vector. We first study the benefit of the biased OFUL algorithm in terms of regret minimization. We then propose two strategies to estimate the bias within the learning-to-learn setting. We show both theoretically and experimentally, that when the number of tasks grows and the variance of the task-distribution is small, our strategies have a significant advantage over learning the tasks in isolation.
\end{abstract}
\section{Introdution} \label{Sec:Introduction}
The multi-armed bandit (MAB)   is a simple framework formalizing the online learning problem constrained to partial feedback \citep[see][and references therein]{lattimore2018bandit, UCB1,RobbinsReviewed,RobbinsOriginal, MAB,RegretMAB}. In the last decades it has receiving increasing attention due to its wide practical importance and the
theoretical challenges in designing principled and efficient learning algorithms.
In particular, applications 
range from recommender systems \cite{li2010contextual, cella2019stochastic,bogers2010movie}, to clinical trials \cite{villar2015multi}, and to adaptive
routing \cite{awerbuch2008online}, among others.

In this paper, we are concerned with linear bandits \citep{OFUL,LinUCB,LinRel}, a consolidated MAB setting in which each arm is associated with a vector of features and the arm payoff function is modeled by a (unknown) linear regression of the arm feature vector. Our study builds upon the OFUL algorithm introduced in \citep{OFUL}, which in turned improved the theoretical analysis initially investigated in \citep{LinUCB, LinRel}. Nonetheless, it may still require a long exploration in order to estimate well the unknown linear regression vector. An appealing approach to solve this bottleneck 
is to leverage already completed tasks by transferring the previously collected experience to speedup the learning process. This framework finds its most common application in the recommendation system domain, where we wish to recommend contents to a new user by matching his preference. Our objective is to rely on past interactions corresponding to navigation of different users to speedup the learning process.
    
\vspace{.2truecm}
\noindent {\bf Previous Work.} During the past decade, there have been numerous theoretical investigation of transfer learning, with a particular 
attention to the problems of multi-task (MTL) \citep{MTLBatch, MTLExcessRIsk, MTLSparse, MTLBenefit, MTLLinear} and learning-to-learn (LTL) or meta-learning \citep{LTLBaxter, LTLRegret, LTLAroundCommonMean, LTLStatGuarantees, LTLSGD, LTLWeightedMajority}. The main difference between these two settings is that MTL aims to solve the problem of learning well on a prescribed set of tasks (the learned model is tested on the same tasks used during training), whereas LTL studies the problem of selecting a learning algorithm that works well on tasks from a common environment (i.e. sampled from a prescribed  distribution), relying on already completed tasks from the same environment \citep{LTLWeightedMajority, balcan2019provable, LTLAroundCommonMean,  LTLSGD}. In either case the base tasks considered have always been supervised learning ones. Recently, the MTL setting has been extended to a class of bandit tasks, with encouraging empirical and theoretical results  \citep{SeqTLMAB,SparseTLRL,CausalTLMAB,MTLContMAB,CrossDomain}, as well as to the case where tasks belong to a (social) graph, a setting that is usually referred to as~\textit{collaborative linear bandit} \citep{GOB,MTLinMAB,CLUB,CAB}. Differently from these works, the principal goal of this paper is to investigate the adoption of the meta-learning framework, which has been successfully considered within the supervised setting setting, to the setting of linear stochastic bandits.


\vspace{.2truecm}
\noindent 
{\bf Contributions.} Our contribution is threefold. First, we introduce in Section \ref{Sec:RegOFUL} a variant of the OFUL algorithm in which the regularization term is modified by introducing a bias vector, 
analyzing the impact of the bias in terms of regret minimization. Second, and more importantly, in Sections \ref{Sec:SolutionA} and \ref{Sec:SolutionB} we propose two alternative approaches to estimate the bias, within the meta-learning setting. We establish theoretical results on the regret of these methods, 
highlighting that, when the task-distribution has a small variance and the number of tasks grows, adopting the proposed meta-learning methods lead a substantial benefit in comparison to using the standard OFUL algorithm. Finally, in Section \ref{Sec:Experiments} we compare experimentally the proposed methods with respect to the standard OFUL algorithm on both synthetic and real data.

\section{Learning Foundations}\label{Sec:Preliminaries}
    In this section we start by briefly recalling the standard stochastic linear bandit framework and we then present the considered LTL setting.
    \subsection{Linear Stochastic Bandits} 
    Let $T$ be a positive integers and let $[T]=\{1,\dots,T\}$. A Linear Stochastic MAB is defined by a sequence of $T$ interactions between the agent and the environment. At each round $t\in[T]$, the learner is given a decision set $\mathcal{D}_t\subseteq \mathbb{R}^d$ from which it has to pick an arm $\mathbf{x}_{t}\in\mathcal{D}_t$. Subsequently, it observes the corresponding reward $y_t = \mathbf{x}_{t}^\top \w^* + \eta_t$ which is defined by a linear relation with respect to an unknown parameter $\w^*\in\mathbb{R}^d$ combined with a sub-gaussian random noise term $\eta_t$. Thanks to the knowledge of the true parameter $\w^*$, at each round $t$ the optimal policy picks the arm $\mathbf{x}^*_t = \arg\max_{\mathbf{x}\in D_t} \mathbf{x}^\top \w^*$, maximizing the instantaneous reward. The learning objective is to maximize the cumulative reward, or equivalently, to minimize the \textit{pseudo-regret}
    \begin{equation}
    	R(T, \w^*) = \sum_{t=1}^{T}(\mathbf{x}^*_t - \mathbf{x}_{t})^\top \w^*. \nonumber
    \end{equation}
    As learning algorithm we consider OFUL \citep{OFUL}. At each round $t\in [T]$, it estimates $\w^*$ by ridge-regression over the observed arm reward pairs, that is, 
    \begin{align}\label{Eq:RidgeRegression}
    		\widehat{\w}^\lambda_t = \arg\min_{\w \in \mathbb{R}^d} \norm{\mathbf{X}_{t} \w - \mathbf{y}_{t}}_2^2 + \lambda \norm{\w}_2^2
    \end{align}	
    where $\mathbf{X}_{t}$ is the matrix whose rows are $\mathbf{x}_1^\top,\dots,\mathbf{x}_t^\top$, $\mathbf{I}$ is the $d\times d$ identity matrix and $\mathbf{y}_t = (y_1,\dots,y_t)^\top$. A key insight behind OFUL is to update online  
    a confidence interval $\mathcal{C}_t$ containing the true parameter $\w^*$ with high probability and centered in $\widehat{\w}^\lambda_t$. According to Theorem 2 of \citep{OFUL}, assuming that  $\|\w^*\|_2\leq S$ and $\|\mathbf{x}\|_2 \leq L$, for every $\mathbf{x} \in \cup_{s=1}^{t} \mathcal{D}_s$, then for any $\delta>0$, with probability  at least  $1-\delta$, for every $t\geq 0$, $\w^*$ lies in
	\begin{equation}\label{Eq:C_t}
		\mathcal{C}_t (\delta) = \bigg\{ \w \in \mathbb{R}^d: \norm{\widehat{\w}^\lambda_t  - \w }_{\mathbf{V}^\lambda_t} \leq R \sqrt{d\log \frac{1 + t L^2/\lambda}{\delta}} + \lambda^{\frac{1}{2}}S =: \beta^\lambda_t(\delta) \bigg\}
	\end{equation}
    where $\mathbf{V}^\lambda_t = \lambda \mathbf{I} + \mathbf{X}_t^\top \mathbf{X}_t$. According to the {optimism in the face of uncertainty principle}, at each round $t$ OFUL picks the arm $\mathbf{x}_t$ by solving the following optimization problem:
    \begin{equation}
        \mathbf{x}_t = \arg\max_{\mathbf{x} \in \mathcal{D}_t} \max_{\wtilde^\lambda_{t} \in \mathcal{C}_t} \mathbf{x}^\top \wtilde_{t}.
    \end{equation}
As was proved in Lemma 5 of \citep{SOFUL}, this corresponds to choose the input 
    \begin{equation}\label{Eq:OFU}
		\mathbf{x}_t \in \arg\max_{\mathbf{x}\in\mathcal{D}_t} \Big\{ \mathbf{x}^\top \widehat{\w}^\lambda_{t-1} + \beta_{t-1}(\delta)\norm{\mathbf{x}}_{\left(\mathbf{V}^\lambda_{t-1}\right)^{-1}} \Big\}.
	\end{equation}
	Finally, with probability at least $1- \delta$, OFUL satisfies (see Theorem 3 of \cite{OFUL}) the upper bound
	\begin{equation}
		R(T, \w^*) \leq 4 \sqrt{Td\log\Big(1 + \frac{TL}{\lambda d}\Big)}\bigg(\lambda^\frac{1}{2}S + R \sqrt{2\log(1/\delta) + d \log(1 + TL/(\lambda d))}\bigg). \nonumber
	\end{equation}
    We can now formally introduce the considered LTL learning framework for the family of tasks we analyze in this work: biased regularized linear stochastic bandits. 
    \subsection{LTL with Linear Stochastic Bandits.} We assume that each learning task $\w\in\mathbb{R}^d$ representing a linear bandit, is sampled from a task-distribution $\rho$ of bounded support in  $\mathbb{R}^d$. The objective is to design a meta-learning algorithm which is well suited to the environment. Specifically, we assume to receive a sequence of tasks $\w_1,\dots,\w_N,\dots$ which are independently sampled from the task-distribution (\textit{environment}) $\rho$. Due to the interactive nature of the bandit setting, we do not have any prior information related to a new task; we collect information about it along the interaction with the environment. After completing the $j$-th task, we store the whole interaction in a dataset $Z_j$ which is formed by $T$ entries $(\mathbf{x}_{j,t}, y_{j,t})^T_{t=1}$. Clearly, the dataset entries are not i.i.d sampled from a given distribution, but each dataset $Z_j$ corresponds to the recording of the learning policy in terms of the arm $\mathbf{x}_{j,t}$ picked from the decision set $\mathcal{D}_t^j$ and its corresponding reward $y_{j,t}$ while facing the task specified by the unknown vector $\w_j$. Starting from these datasets, we wish to design an algorithm $\mathcal{A}$ which suffers a low regret on a new task $\w_{N+1}\sim\rho$. This can be stated into requiring that $\mathcal{A}$ trained over $N$ datasets has small \textit{transfer-regret}:
    \begin{equation}
    	\mathcal{R}(T,\rho) = \mathbb{E}_{\w\sim\rho}\Big[\mathbb{E}\big[ R(T,\w)\big]\Big] \nonumber
    \end{equation}
    where the inner expectation is with respect to rewards realizations due to their noisy components.
\section{Biased Regularized OFUL}\label{Sec:RegOFUL}
    We now introduce BIAS-OFUL, a biased version of OFUL, which is instrumental for our meta-learning setting. 
    Although not feasible, the proposed algorithm serves as a basis to study the theoretical properties of meta-learning with stochastic linear bandit tasks. In Section \ref{Sec:Experiments} we will present a more practical version of it.
    \paragraph{Regularized Confidence Sets} The idea of following a bias in a specific family of learning algorithms is not new in the LTL literature \citep{LTLAroundCommonMean,LTLSGD}.
    Inspired by \citep{LTLSGD} we modify the regularization in the computation of the confidence set centroid $\what^\lambda_t$, where the regularization is now defined as a square euclidean distance to the bias parameter $\mathbf{h}\in\mathbb{R}^d$. Given a fixed vector $\mathbf{h}$, at each round $t\in[T]$ BIAS-OFUL estimates the regularized centroid of the confidence ellipsoid as
    \begin{equation}
    	\what^\mathbf{h}_t = \arg\min_{\w}\norm{\mathbf{X}^\top_{t} \w - \mathbf{Y}_{t}}_2^2 + \lambda \norm{\w - \mathbf{h}}_2^2 \nonumber
    \end{equation}
    whose solution is given by
    \begin{equation}\label{Eq:BiasedParameterEstimate}
    	\what^\mathbf{h}_t = \left(\mathbf{V}^\lambda_t\right)^{-1} \mathbf{X}_t^\top (\mathbf{Y}_t - \mathbf{X}_t \mathbf{h}) + \mathbf{h}.
    \end{equation}
    This result follows directly from the standard ridge-regression by making the change of variable $\mathbf{w} = \mathbf{v} + \mathbf{h}$ and then solving for $\mathbf{v}$. \\
    As we have mentioned in the previous section, at each round $t$ OFUL keeps also updated a confidence interval $\mathcal{C}_t$ (see Equation \ref{Eq:C_t}) centered in $\what^\lambda_t$ which contains $\w^*$ with high probability. We now derive a confidence set for the biased regularized estimate $\mathbf{\widehat{w}}^\mathbf{h}_t$, assuming that we have access to an oracle to compute the distance $\norm{\mathbf{h} - \w^*}_2$. This seems quite restrictive, however later in the paper we will show how leveraging similar related tasks we can exploit this bound to take advantage of the bias version of OFUL, without having to know the above distance a-priori.
    \begin{theorem}\label{Th:BiasedConfidenceSet}
    	Assuming $\norm{\mathbf{h}}_2\leq S$, $\norm{\w^*}_2\leq S$ and $\norm{\mathbf{x}}_2 \leq L$ $\forall \; \mathbf{x} \in \cup_{s=1}^{t} \mathcal{D}_s$, then for any $\delta>0$, with probability at least $1-\delta$, $\forall t\geq 0$, $\w^*$ lies in the set
    	\begin{equation}\label{Eq:BiasedConfidenceSet}
    		\mathcal{C}^{\mathbf{h}}_t(\delta) = \Bigg\{ \w \in \mathbb{R}^d: \norm{\what^\mathbf{h}_t  - \w }_{\mathbf{V}^\lambda_t} \leq \lambda^{\frac{1}{2}}\norm{\mathbf{h} - \w^*}_2 +\nonumber + R \sqrt{2  \log\Bigg( \frac{\det\left(\mathbf{V}^\lambda_t\right)^{1/2}}{\det\left(\lambda I\right)^{1/2}\delta} \Bigg)}= \beta^\mathbf{h}_t(\delta) \Bigg\}.
    	\end{equation}
    \end{theorem}
\noindent    The proof can be found in the appendix material. We will now study the impact of the bias $\mathbf{h}$ in terms of regret.\\
    
\subsection{Regret Analysis with Fixed Bias}\label{SubSec:RegretFixedBias}
Given the confidence set defined in Theorem \ref{Th:BiasedConfidenceSet} and the \textit{optimism principle} translated into selecting the next arm according to Equation \ref{Eq:OFU}, we can analyze the expected pseudo-regret depending on the value of $\mathbf{h}$.

\begin{lemma}\label{Lemma:RegretFixedBias} (REG-OFUL Expected Regret) Under the same assumptions of Theorem \ref{Th:BiasedConfidenceSet}, if in addition, 
for all $t$ and all $\mathbf{x}\in\mathcal{D}_t$, $\mathbf{x}^\top \mathbf{w}^* \in [-1,1]$, and considering $\lambda \geq 1$, we have:
\begin{align*}
    \overline{R}(T,\mathbf{w}^*) &= \mathbb{E}\left[R(T,\mathbf{w}^*)\right] \\
    &\leq C\sqrt{Td\log\left(1+\frac{TL}{\lambda d}\right)}\Bigg(\lambda^\frac{1}{2} \norm{\mathbf{w}^* - \mathbf{h}}_2 + R\sqrt{d \log(T + T^2 L/(\lambda d))}\Bigg)
\end{align*}
where the expectation is respect to the reward generation and $C>0$ is a constant factor.
\end{lemma}

We now analyze the regret for two different values of $\mathbf{h}$. In particular we wish to highlight how setting a good bias can speedup the process of learning with respect to using the standard OFUL approach \citep{OFUL}.

\begin{corollary}\label{Cor:RightH} Under the conditions of Lemma \ref{Lemma:RegretFixedBias}, the following bounds on the expected regret of BIAS-OFUL holds:
	\begin{itemize}
	\vspace{-.3truecm}
		\item[{\rm (i)}] {\rm Independent Task Learning (ITL)}, given by setting $\mathbf{h} = \mathbf{0}$ satisfies the following expected regret bound
			\begin{equation*}
			    \overline{R}(T,\mathbf{w}^*) \leq C \sqrt{Td\log\bigg(1 + \frac{TL}{ \lambda d}\bigg)}\bigg(\lambda^\frac{1}{2}S+ R\sqrt{d \log(T + T^2 L/(\lambda d))}\bigg) 
			\end{equation*}	
			which is of order $\bigO(d\sqrt{T})$ for any $\lambda\geq1$.
		\item[{\rm (ii)}] {\rm The Oracle}, given by setting $\mathbf{h}=\mathbf{w^*}$
		satisfies 
			\begin{equation*}
    			\overline{R}(T,\mathbf{w}^*) \leq C \sqrt{Td\log\bigg(1 + \frac{TL}{ \lambda d}\bigg)} \bigg( R\sqrt{d \log(T + T^2 L/(\lambda d))}\bigg)
			\end{equation*}
				which is $0$ as $\lambda\to\infty$.
	\end{itemize}
\end{corollary}
The proofs can be found in the supplementary material. The main intuition is that, as long as we can set $\h=\w^*$, the bigger the the regularization parameter $\lambda$ is, the more the Oracle policy tends to select the arm only based on $\w^*$, thereby becoming equivalent to the optimal policy.

\subsection{Transfer Regret Analysis with Fixed Bias}\label{SubSec:TransferRegretFixedBias}

Following the above analysis for the single task case, we now study the impact of the bias in the transfer regret bound. To this end, we introduce the variance and the mean absolute distance of a bias vector $\mathbf{h}$ relative to the environment of task,
\[
    \var_{\mathbf{h}} = \mathbb{E}_{\w\sim\rho}\big[\norm{\w-\mathbf{h}}_2^2\big], \quad \mar_{\mathbf{h}} = \mathbb{E}_{\w\sim\rho}\big[\norm{\w-\mathbf{h}}_2\big]
\]
and we observe that $\wbar = \mathbb{E}_{\w\sim\rho} \w = \arg\min_{\mathbf{h}\in\mathbb{R}^d}\var_{\mathbf{h}}$ and $\mathbf{m} = \arg\min_{\mathbf{h}\in\mathbb{R}^d}\mar_{\mathbf{h}}$. With this in hand, we can now analyze how the transfer regret can be upper bounded as a function of the introduced terms. 
\begin{lemma}\label{Lemma:TransferRegretBound}(Transfer Regret Bound) Under the same conditions in Theorem \ref{Th:BiasedConfidenceSet} and Lemma \ref{Lemma:RegretFixedBias}, the expected transfer regret of BIAS-OFUL can be upper bounded as:
    \begin{align*}
        \mathcal{R}(T,\rho) &\leq C \sqrt{Td \lambda \log\left(1 + \frac{TL}{\lambda d}\right) } \mar_\mathbf{h} + RCd\sqrt{T \log\bigg(T + \frac{T^2L}{\lambda d}\bigg)\log\bigg(1+\frac{TL}{\lambda d}\bigg)}\\
        &\leq C \sqrt{Td \lambda \log\left(1 + \frac{TL}{\lambda d}\right) \var_\mathbf{h}} + RCd\sqrt{T \log\bigg(T + \frac{T^2L}{\lambda d}\bigg)\log\bigg(1+\frac{TL}{\lambda d}\bigg)}
    \end{align*}
\end{lemma}
\begin{proof}
The first statement is the expectation with respect to the \textit{task}-distribution $\rho$ applied to Lemma \ref{Lemma:RegretFixedBias}, while the second follows by applying Jensen's inequality.
\end{proof}
We can now replicate what we have done in Corollary \ref{Cor:RightH} and consider the transfer regret bound for two different values of the hyper-parameter $\mathbf{h}$. The main difference is that here, there is not an a-priori correct value for $\mathbf{h}$ as it depends on the task-distribution $\rho$.
\begin{corollary}\label{Cor:RightHTransfer} Under the same assumptions in  Theorem \ref{Th:BiasedConfidenceSet} and Lemma \ref{Lemma:RegretFixedBias}, and setting $\lambda = \frac{1}{T \Var_\mathbf{h}}$, the following bounds on the transfer regret hold
\vspace{-.3truecm}
	\begin{itemize}
			\item[{\rm (i)}] {\rm Independent Task Learning (ITL)}, given by setting the bias hypeparameter $\mathbf{h}$ equal to $\mathbf{0}$, satisfies
		\begin{equation*}
		    \mathcal{R}(T,\rho) \leq \Bigg[1 + \sqrt{T d \log\bigg(T + \frac{T^3L\var_{\mathbf{0}}}{ d}\bigg)}\Bigg]C\sqrt{d \log\bigg(1+\frac{T^2L\var_{\mathbf{0}}}{ d}\bigg)}
		\end{equation*}
		\item[ {\rm (ii)}] {\rm The Oracle}, given by setting the bias hyperparameter $\mathbf{h}$ equal to the mean task $\wbar$, satisfies
		\begin{equation*}
		    \mathcal{R}(T,\rho) \leq \Bigg[1 + \sqrt{T d \log\bigg(T + \frac{T^3L\var_{\wbar}}{ d}\bigg)}\Bigg] C\sqrt{d \log\bigg(1+\frac{T^2L\var_{\wbar}}{ d}\bigg).}
		\end{equation*}
	\end{itemize}
\end{corollary}
\begin{proof}
These results directly follow from Lemma \ref{Lemma:TransferRegretBound}. We have picked $\lambda=\frac{1}{T\Var_\mathbf{h}}$ in order to highlight the multiplicative term $\log(1+\Var_\mathbf{h})$ which tends to zero according to the variance $\Var_\mathbf{h}$ of the task-distribution $\rho$.
\end{proof}
Therefore, running BIAS-OFUL with bias $\mathbf{h}$ equal to $\wbar$ brings a substantial benefit with respect to the unbiased case when the second moment of the task-distribution $\rho$ is much bigger than its variance. Specifically, we introduce the following assumption.
\begin{assumption}\label{Ass:LowVariance}(Low Biased Variance)
    \begin{equation}
        \var_\wbar = \mathbb{E}_{\w\sim\rho} \norm{\w-\wbar}_2^2 \ll \mathbb{E}_{\w\sim\rho} \norm{\w}_2^2 = \var_\mathbf{0}.
    \end{equation}
\end{assumption}
Notice also that the choice $\lambda = 1/(T\var_\mathbf{h})$, implies that, as $\var_\wbar$ tends to $0$, the regret upper bound of the oracle case tends to zero too reflecting the result of Corollary \ref{Cor:RightH}. More in general, we can state that when the environment (i.e. the task-distribution $\rho$) satisfies Assumption \ref{Ass:LowVariance}, leveraging on tasks similarity would gives a substantial benefit compared to learning each task separately. Since in practice the mean task parameter $\wbar$ is unknown, in the following sections we propose two alternative approaches to estimate $\wbar$.
\section{A High Variance Solution}\label{Sec:SolutionA}
\begin{algorithm}[!t]
\caption{Within Task Algorithm: BIAS-OFUL}
\label{Alg:REG-OFUL}
\begin{algorithmic}[1]
\REQUIRE{$\lambda > 0, \hhat_0 \in \reals^d$}
\STATE $\what_0^\h = \hhat_0, \V_0^{-1} = \frac{1}{\lambda} \I$.
\FOR{$t = 1$  {\bfseries to} $T$}
\STATE GET decision set $D_t$
\STATE SELECT $\x_t \in D_t$ with bias $\h=\hhat^\lambda_{j,t}$
\STATE OBSERVE reward $y_t$
\STATE UPDATE $\V_t = \V_{t-1} + \x_t \xtop_t$
\STATE UPDATE $\hhat_t$ according to the meta-algorithm
\STATE UPDATE $\what_t^\mathbf{h}$ using Equation \ref{Eq:BiasedParameterEstimate}
\ENDFOR
\end{algorithmic}
\end{algorithm}
\begin{algorithm}[!t]
\caption{Meta-Algorithm: Estimating $\hhat^\lambda$}
\label{Alg:META}
\begin{algorithmic}[1]
\FOR{$j = 1$ {\bfseries to} $N$}
    \STATE SAMPLE new task $\w_j \sim \rho$
    \STATE SET $\hhat^\lambda_{j,0}$
    \STATE RUN Algorithm \ref{Alg:REG-OFUL} with parameter $\hhat^\lambda_{j,0}$
\ENDFOR
\end{algorithmic}
\end{algorithm}
In this section, we present our first meta-learning method. We begin by introducing some additional notation. We let $\xh_{j,t}$ be the arm pulled by the BIAS-OFUL algorithm (Algorithm \ref{Alg:REG-OFUL}) at round $t$-th of the $j$-th task. We denote by $\V_{j,T} = \sum_{s=1}^T \xh_{j,s} \xhtop_{j,s} $ the design matrix computed with the $T$ arms picked during the $j$-th task. For each terminated task $j\in[N]$ we also define $\mathbf{b}_{j,T} = \mathbf{X}_{j,T}^\top \mathbf{Y}_{j,T}$. Finally, we introduce the \textit{mean estimation error}
\[
    \epsilon_{N,t}(\rho) = \norm{\wbar - \hhat^\lambda_{N,t}}_2^2
\]
which is the error of our estimate $\hhat^\lambda_{N,t}$ with respect to the true mean task $\wbar$, at round $t$ of the $N+1$-th task.
\subsection{Averaging the Estimated Task Parameters}
An intuitive solution to bound the estimation error $\epsilon_{N,t}$ is to simply average of the estimated task parameters $\what^\lambda_j$ computed according to Equation $\ref{Eq:RidgeRegression}$ on the dataset $Z_j$ without considering any bias.
\begin{equation}\label{Eq:ApproachA}
    \hhat_{N,t}^\lambda = \frac{1}{NT + t} \Bigg( \sum_{j=1}^{N} T \what^\lambda_{j,T} + t \what^\lambda_{N+1, t} \Bigg).
\end{equation}
By adopting this approach, we have the following bound on the transfer regret.
\begin{theorem}\label{Th:TransferRegretBoundA}(Transfer Regret Bound).
    Let the assumptions of Lemma \ref{Lemma:TransferRegretBound} hold and let $\hhat_{N,t}^\lambda$ be defined as in Equation \eqref{Eq:ApproachA}. Then, it hold that
        \begin{align*}
        \mathcal{R}&(T,\rho) \leq d C\sqrt{T \log\left(1 + \frac{T^2L \bigg(\Var_\wbar + \epsilon_{N,T}(\rho) \bigg)}{d}\right)}
    \end{align*}
    where the mean estimation error can be bound as
    \begin{equation*}
        \sqrt{\epsilon_{N,T}(\rho)} \leq H_\rho(N+1, \wbar) + \max_{j=1,\dots,N} \frac{\beta^\lambda_j\big(1/T\big)}{\lambda^{1/2}_{\min}(\mathbf{V}^\lambda_{j,T})}.
    \end{equation*}
    Here, $\beta^\lambda_j\big(\frac{1}{T}\big)$ refers to the confidence interval computed with OFUL (see Equation \ref{Eq:C_t}) and $H_\rho(N+1, \wbar)=\norm{\wbar - \Hbar_{N, t}}_2$ with $\Hbar_{N,t+1} = \frac{1}{NT+t} \big( \sum_{j=1}^{N} T \mathbf{w}_j + t \w_{N+1} \big)$.
\end{theorem}
\begin{proof}
We follow the reasoning in Corollary \ref{Cor:RightHTransfer}, this time setting $\mathbf{h}= \hhat^\lambda_{N,T}$, and then observe that
    \begin{align*}
    \sqrt{\epsilon_{N,T}(\rho)} &= \norm{\wbar - \hhat^\lambda_{N,T}}_2 \leq \norm{\wbar - \Hbar_{N, T}}_2 + \norm{\Hbar_{N,T} - \hhat^\lambda_{N,T}}_2\\
    &\quad\quad= H_\rho(N+1, \wbar) + \norm{\Hbar_{N,T} - \hhat^\lambda_{N,T}}_2\\ 
    &\quad\quad\leq H_\rho(N+1, \wbar) + \max_{1\leq j \leq N+1}\norm{\w_j - \widehat{\w}^\lambda_{j,T}}_2\\
    &\quad\quad\leq H_\rho(N+1, \wbar) + \max_{1\leq j \leq N+1}\frac{\norm{\w_j - \widehat{\w}^\lambda_{j,T}}_{\mathbf{V}^\lambda_{j,T}}}{\lambda^{1/2}_{\min}(\mathbf{V}^\lambda_{j,T})}\\
    &\quad\quad\leq H_\rho(N+1, \wbar) + \max_{1\leq j \leq N+1}\frac{\beta^\lambda_j\big(1/T\big)}{\lambda^{1/2}_{\min}(\mathbf{V}^\lambda_{j,T})}. 
\end{align*}
\end{proof}
The term $H_\rho(N+1, \wbar)$ denotes the estimation error of the empirical mean computed from the $N+1$ tasks vectors  $(\w_j)_{j=1}^{N+1}$, relative to the true mean $\wbar$. Since the $\w_j$ are independent random $d$-dimensional vectors drawn from $\rho$ we can apply the following vectorial version of the Bennett's inequality \citep[see, e.g.,][Lemma 2]{smale2007learning}.
\begin{lemma}\label{Lemmma:VecBennett}
    Let $\w_1,\dots,\w_N$ be N independent random vectors with values in $\mathbb{R}^d$ sampled from the task-distribution $\rho$. Assuming that $\forall j \in [N]: \norm{\w_j}\leq S$, then for any $0<\delta<1$, it holds, with probability at least $1-\delta$
    \begin{equation*}
        H(N,\wbar) \leq \frac{2\; \log(2/\delta)\;S}{N} + \sqrt{\frac{2\; \log(2/\delta)\;\Var_\mathbf{0}}{N}}.
    \end{equation*}
\end{lemma}
The above lemma says that the error $H_\rho(N,\wbar)$ goes to zero as $N$ grows to infinity. Therefore the estimation error $\epsilon_{N,t}(\rho)$ is dominated by the ``variance'' term 
$\max_{1\leq j\leq N} \beta^\lambda_j\big(1/T\big)\lambda^{-1/2}_{\min}(\mathbf{V}^\lambda_{j,T})$, associated with the worst past task. By relying on linear regression results \cite{StochasticRidgeRegression} we have that $\lambda_{\min}(\V_j)\geq\log T$. Moreover, as $\lambda_{\min}(\V_j^\lambda)\geq\lambda + \lambda_{\min}(\V_j)$, we observe an increasing sensitivity of the incurred variance to the $\lambda$ parameter for small value of $T$. Finally, according to our choice of $\lambda=1/T\Var_{\hhat^\lambda}$, the suffered variance increases with the variance of our estimator. The latter in turns increases with the variance of the distribution $\rho$, which corresponds to the case in which Assumption \ref{Ass:LowVariance} tends to be violated.
\section{A High Bias Solution}\label{Sec:SolutionB}
In this section we will present an alternative estimator of the true mean $\wbar$, which is inspired by the existing multi-task bandit literature \cite{CLUB, CAB, MTLinMAB}. This estimator exploits together all the samples associated to the past tasks $Z_1,\dots,Z_N$, with the aim of reducing the variance. This is unlike the previous estimator which separately considers the ridge-regression estimates $\what_1,\dots,\what_N$ in Equation \ref{Eq:ApproachA}. As we will see, this approach will reduce the variance but it will introduce an extra-bias. Before presenting this second approach we require some more notation. We let $\Vtilde_{N,t} = \sum_{j=1}^{N} \mathbf{V}_{N,T} + \V_{N+1, t}$ the global design matrix containing the design matrices associated to past tasks $\V_{1,T},\dots,\V_{N,T}$ and the current design matrix $\V_{N+1,t}$. Analogously $\btilde_{N,t} = \sum_{j=1}^{N} \mathbf{b}_{j,T} + \mathbf{b}_{N+1,t}$ refers to global counterpart of $\mathbf{b}_{j,t}$. We denote with $\left|A\right| =  \sup\{\norm{A\mathbf{\mathbf{x}}}: \mathbf{x} \in \reals^d ,\norm{\mathbf{x}} = 1\}$ the norm of matrix A induced by the norm $\norm{\cdot}$, which if no specified is the Euclidean norm. Finally, we denote with $\sigma_{\max}(\mathbf{A})$ the biggest singular value associated with matrix $\mathbf{A}$.
\subsection{Global Ridge Regression}
In order to reduce the variance, our second approach estimates, at each round $t$ of the new sampled task $N+1$, the mean task $\wbar$ as a \textit{global ridge regression} computed over all the available samples as
\begin{equation}\label{Eq:GlobalRR}
    \hhat^\lambda_{N,t} = \left(\Vtilde^\lambda_{N,t-1}\right)^{-1}\btilde_{N,t-1}.
\end{equation}
Our next result provides a bound on the transfer regret of this proposed strategy. The proof is presented in Section~\ref{Sec:AppendixSolutionB} of the appendix. 
\begin{theorem}\label{Th:TransferRegretBoundB}(Transfer Regret Bound). Let the assumptions of Lemma \ref{Lemma:TransferRegretBound} hold and let $\hhat^\lambda_{N,t}$ be defined as in Equation \eqref{Eq:GlobalRR}. Then, the following upper bound holds
\begin{equation*}
    \mathcal{R}(T,\rho) \leq dC \sqrt{T \log\left(1 + \frac{T^2L \bigg(\Var_\wbar + \epsilon_{N,t}(\rho) \bigg)}{d}\right)}
\end{equation*}
where the mean estimation error can be bound as
\begin{align*}
    \sqrt{\epsilon_{N,T}(\rho)} &\leq 
    \frac{S}{\lambda {+} \nu_{\min}} + 2(N{+}1) \max_{1\leq j\leq N{+}1} \widetilde{H}(N{+}1, \w_j)\\
    &+ R\sqrt{\frac{2}{\lambda {+} \nu_{\min}}\log\bigg(T\bigg(1 + \frac{N T L^2}{\lambda d}\bigg)\bigg)} + H_\rho(N {+} 1,\wbar)
\end{align*}
and defined $\nu_{\min} = \lambda_{\min}(\Vtilde_{N,T})$ and we introduced \[
\Htilde(N,\w_j) =  H_\rho(j,\w_j) \sigma_{\max} \Big(\mathbf{V}_{j,T}\Vtilde^{-1}_{N,T}\Big)\]
which is a weighted form of the 
estimation error $H_\rho(j,\w_j)$ towards the current task vector $\w_j$, where the weights are defined in terms of tasks misalignment $\sigma_{\max} \big(\mathbf{V}_{j,T}\Vtilde^{-1}_{N,T}\big)$.
\end{theorem} 
The previous variance term $\frac{\beta^\lambda_j(1/T)}{\lambda_{\min}(\mathbf{V}^\lambda_{j,T})}$ has been now replaced by $\frac{\beta^\lambda(1/NT)}{\lambda+\nu_{\min}}$. It should be easy to observe that $\nu_{\min} \geq \frac{N}{d}\lambda_{\min}(\V_j) \; \forall j \in [N]$ which leads a reduction of factor $d/N$ to the variance, which goes to zero as $N$ goes to infinity. This gain does not come for free, in fact this approach introduces a potentially high bias: $2(N+1) \max_{j=1,\dots,N+1} \widetilde{H}(N+1, \w_j)$ which increases with the tasks misalignment $\sigma_{\max} \big(\mathbf{V}_{j,T}\Vtilde^{-1}_{N,T}\big)$.
\subsection{Tasks Misalignment}
We now analyze the tasks misalignment factors appearing in Theorem \ref{Th:TransferRegretBoundB}, namely, the quanitities $\sigma_{\max} \big(\mathbf{V}_{j,t}\Vtilde^{-1}_{N,t}\big)$ and $\Htilde(N,\w_j)$. For this purpose, we consider two opposite environments of tasks.

In the first case we assume that all the tasks parameters are equal to each other and far from the zero $d$-dimensional vector. This scenario, which corresponds to put all the mass of the task-distribution $\rho$ on a single task parameter $\wbar$, is clearly in agreement with Assumption \ref{Ass:LowVariance}. We  expect this to be the most favorable scenario, since after completing a task, we face exactly the same task again and again. In this case, independently on the covariance matrices, whose construction also depends on the decision sets available in the different tasks, it is simple to observe that we are not suffering any bias, that is, $\Htilde(N,\w_j)=0 \quad$ for every $j \in [N]$ as all the task parameters are equal to each other.

The second environment is characterized by a task distribution $\rho$ that is unform on finitely many orthogonal tasks.
For instance, this is the scenario when $\rho$ is uniform distributed over the standard basis vectors $\{(S,0,\dots,0),\dots,(0,\dots,0,S)\} \in \mathbb{R}^d$. Differently from the previous scenario, here after completing a task we will probably face an orthogonal task. It should be quite natural to see that this is the most unfavorable case and to expect to not have transfer learning between tasks. This is confirmed by the regret bound due to the misalignment expressed by the covariance matrices $\sigma_{\max} \big(\mathbf{V}_{j,t}\Vtilde^{-1}_{N,t}\big)$. Indeed, since we can have at most $d$ misaligned arms, we have the following upper bound $\frac{d}{N}$ to the term $\sigma_{\max}\big(\mathbf{V}_{j,t}\Vtilde_{N,t}^{-1}\big)$. 
Based on these observations we can conclude that 
the bigger the cardinality of the set of basis induced by the distribution $\rho$, 
the larger the number of completed tasks required to have a proper transfer. We will now focus on an intermediate case satisfying Assumption \ref{Ass:LowVariance}.
In order to control the term $\sigma_{\max} \big(\mathbf{V}_{j,t}\Vtilde^{-1}_{N,t}\big)$ and to give the possibility to generate aligned matrices when dealing with similar tasks, we introduce an additional mild assumption:
\begin{assumption}\label{Ass:FixedDecisionSet}(Shared Induced Basis)
The decision sets are shared among all the tasks and tasks sampled according to Assumption $\ref{Ass:LowVariance}$ induces that the covariance matrices generated by running the BIAS-OFUL algorithm (Algorithm \ref{Alg:REG-OFUL}) share the same basis:
    \begin{equation}
        \mathbf{V}_i = \mathbf{P} \Sigma_i \mathbf{P^*}, \quad \forall i \in [N].
    \end{equation}
\end{assumption}
This assumption is quite mild as it just states that similar tasks share the same pulled arms with no restrictions on the pulling frequency.  This is the case when the decision set is fixed among different rounds and tasks, that is, $\mathcal{D}_{j,t} = \mathcal{D} \; \forall j \in [N]$ and $\forall t \in [T]$, and consists of $d$ orthogonal arms. If Assumption \ref{Ass:FixedDecisionSet} is satisfied, then we can obtain the following bound: $\sigma_{\max} \big(\mathbf{V}_{j,t}\Vtilde^{-1}_{N,t}\big)\leq 1$. Furthermore, if we denote by $M$ the number of tasks necessary to achieve a stationary behavior of the BIAS-OFUL policy in terms of covariance matrices, then 
$\sigma_{\max} \big(\mathbf{V}_{j,t}\Vtilde^{-1}_{N,t}\big)\leq1/(N-M)$. 
\subsection{Smallest Global Eigenvalue $\nu_{\min}$}
It only remains to analyze the term $\nu_{\min}$. We observe that 
it satisfies the lower bound
\begin{equation*}
    \nu_{\min} = \lambda_{\min} \Bigg(\sum_{j=1}^{N+1} \V_{j,T} \Bigg)\geq \sum_{j=1}^{N+1} \lambda_{\min}(\V_{j,T}) \geq (N+1) \log T
\end{equation*}
where in the last step we have relied on linear regression result from \citep{StochasticRidgeRegression} which shows that the
condition $\mathcal{O}(\lambda_{\min})=\log(\lambda_{\max})$ is required to guarantee asymptotic consistency, necessary
to have sublinear anytime regret. Since $\min_{j\in[N]}\lambda_{\max}(V_j) = \mathcal{O}(T)$, this condition implies that $\min_{j\in[N]}\lambda_{\min}(V_j) \geq \log T$.
\section{Experiments} \label{Sec:Experiments}
In this section we test the real effectiveness of the proposed approaches. The theoretical results stated that the method presented in Section \ref{Sec:SolutionA} does not introduce any bias but it may incur an additional variance according to the variance of the task-distribution $\Var_\rho$. On the contrary, the solution proposed in Section \ref{Sec:SolutionB} which massively uses all the observed samples together, reduces the variance (at least) by a factor $d/N$, at the price of an extra bias term.

As it was mentioned in Section \ref{Sec:RegOFUL}, the parameter $\w^*$ associated to each single task is unknown, therefore we cannot compute the gap $\|{\hhat^\lambda - \w^*}\|_2$ defining the term $\beta_t^\mathbf{h}(1/T)$. The main issue is that according to Equation \ref{Eq:OFU}, in order to pick the next arm, it seems that the algorithm needs to compute its exact value. However, we can simply split the norm and rely on the assumption that $\|\w^*\|\leq S$, so to remove the dependency on $\w^*$. Indeed, it is important to emphasize that the real knowledge transfer happens in terms of $\w^\mathbf{h}$, see Equation \ref{Eq:BiasedParameterEstimate}. This can be noticed by observing that the gap $\norm{\hhat^\lambda - \w^*}_2$ equally affects all the available arms.
%
\subsection{Experimental Results}
In all the presented experiments the policy OPT knows the parameter $\w_j$ associated to task $j$ and picks the next arm as $\mathbf{x}_{j,t} = \arg\max_{\mathbf{x}\in D_{j,t}} \mathbf{x}^\top \w_j$. The policies AVG-OFUL and RR-OFUL implement Algorithms \ref{Alg:REG-OFUL} and \ref{Alg:META} and estimate $\hhat$ as per Equations \ref{Eq:ApproachA} and Equation \ref{Eq:GlobalRR}, respectively. The Oracle policy knows the mean task parameter $\wbar$ and uses it as the bias $\h$ in BIAS-OFUL (Corollary \ref{Cor:RightHTransfer} (ii)). Analogously, the ITL policy consists of BIAS-OFUL with bias set equal to $\mathbf{0}$, see Corollary \ref{Cor:RightHTransfer} (i). The regularization  hyper-parameter $\lambda$ was selected over a logarithmic scale. We will start by considering a pair of synthetic experiments in which we show how the hyper-parameter $\lambda$ affects the performance. We then present experiments on two real datasets. We will denote with $K$ the size of the decision set $\mathcal{D}$.

\paragraph{Synthetic Data}
\begin{figure}[t]
    \centering
    \includegraphics[width=0.435\textwidth]{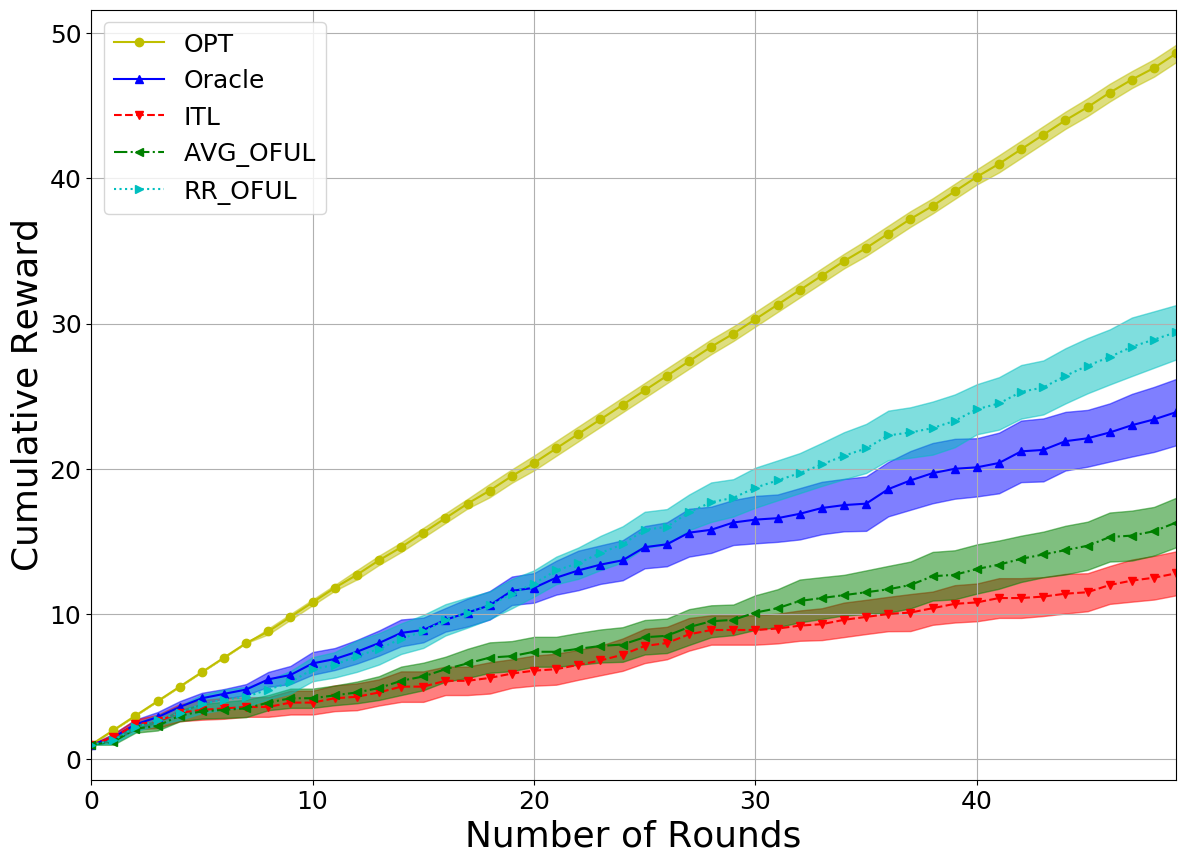}
    \vspace{-.2truecm}
    \caption{Cumulative reward measured after $N=10$ tasks and averaged over $10$ test tasks, with $\lambda=1$.}
    \label{fig:SynthLReg}
\end{figure}
\begin{figure}[t]
    \centering
    \includegraphics[width=0.435\textwidth]{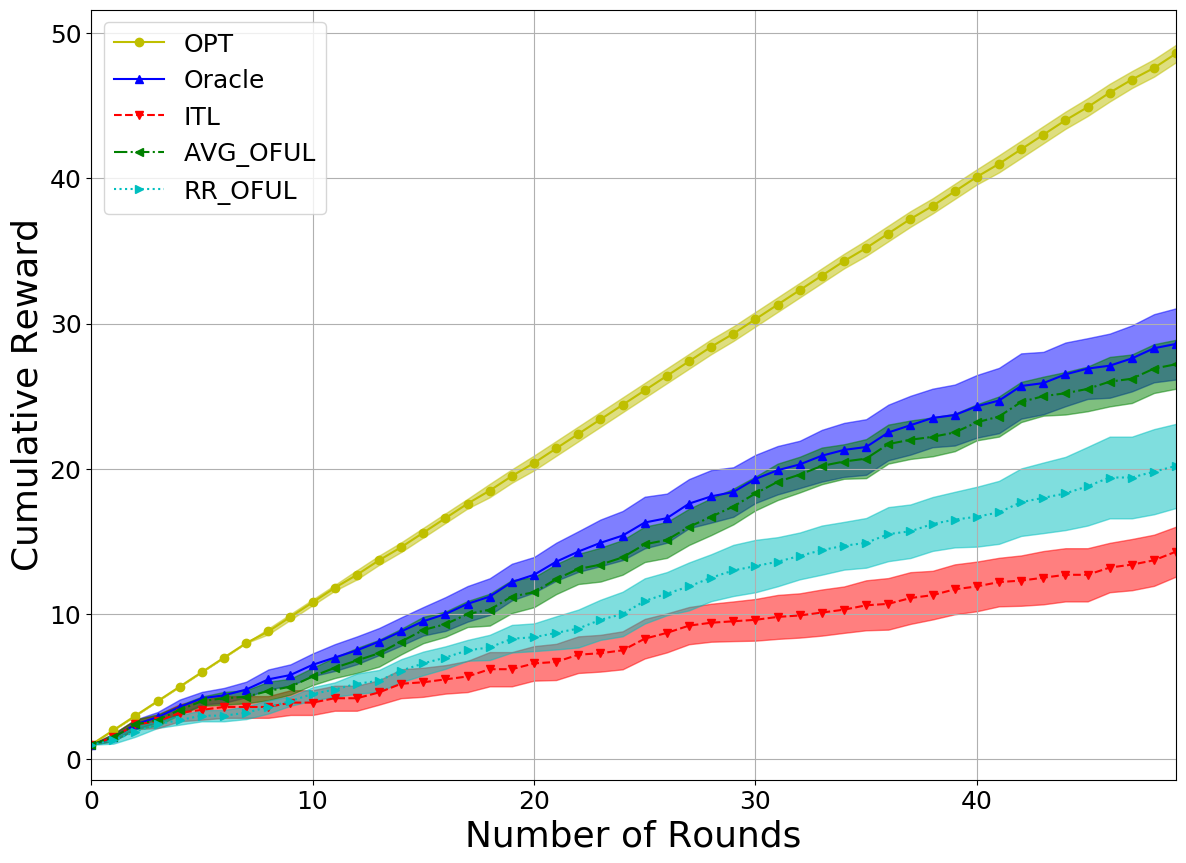}
        \vspace{-.2truecm}
    \caption{Cumulative reward measured after $N=10$ tasks and averaged over $10$ test tasks, with $\lambda=100$.}
    \label{fig:SynthHReg}
\end{figure}
Similarly to what was done in \cite{LTLSGD}, we first generated an environment of tasks in which running the Oracle policy is expected to outperform the ITL approach. In agreement with Assumption \ref{Ass:LowVariance}, we sample the task vectors from a distribution characterized by a much smaller variance than its second moment. That is, each task parameter $\w_j$ is sampled from a Gaussian distribution with mean $\wbar$ given by the vector in $\mathbb{R}^d$ with all components equal to $1$ and $\Var_\rho =  1$. As far as the decision set concerns, we first generate a random square matrix $\mathbf{P}$ with size $d$ and then compute its qr factorization $\mathbf{P}=\mathbf{Q}\mathbf{R}$, where $\mathbf{Q}$ is a matrix with orthonormal columns and $\mathbf{R}$ is an upper-triangular matrix. We then associate to each base arm the direction associated to a column of the matrix $\mathbf{Q}$. This will guarantee having arms that are almost orthogonal each other. Finally, at each round $t \in [T]$ the decision set $\mathcal{D}_t$ is initialized as a set of $K$ random vector that are first shifted towards the respective arm base direction and then normalized. Notice that by following this generation mechanism we avoid any inductive bias between the task vectors and the arms ones, as they are actually independent.  Each task consists of $T=50$ rounds, in which we have $K=5$ arms of size $d=20$. In order to generate the rewards, we first compute the inner product between the user (task) vector and the arm (input) vector, we shift the resulting output interval $[0,1]$ and then add to a Gaussian noise $\mathcal{N}\big(0.5,1\big)$, to compute the rewards.  Finally, we assigned reward $1$ to the arm having the maximum final reward, $0$ to the others. In Figures \ref{fig:SynthLReg} and \ref{fig:SynthHReg}, we report the results generated with $\lambda = 1$ and $\lambda = 100$, respectively. It is easy to observe that the stronger the regularization, the more the AVG-OFUL tends to the Oracle. Conversely, RR-OFUL get penalized with the increasing of $\lambda$, due to its bias. 
\vspace{-.2truecm}
\paragraph{LastFM Data}
The first dataset we considered is extracted from the music streaming service Last.fm \cite{Cantador:RecSys2011} (\textit{http://www.lastfm.com}). It contains 1892 possible users and 17632 artists. This dataset contains information about the artists listened by a given user, and we used this information to define the payoff function. We first removed from the set of items those with less than $30$ ratings and then we repeat the same procedure for the users. This operation yields an user rating matrix of size 741 x 538. Starting from this reduced matrix we derived the arms and the users vectors by computing an SVD decomposition where we kept only the first $d=10$ features associated to the users and to the items. In order to consider tasks satisfying Assumption \ref{Ass:LowVariance}, we randomly pick an user and compute the set of its $N=20$ most similar users according to the l2-distance between their vectors. Each task lasts $T=5$ rounds and consists of $K=5$ arms. At each round $t$, the decision set consists of one arm whose rating was at least equal to $4$ and $K-1$ arms whose ratings were at most equal to $3$. The rewards were then generated analogously to the synthetic case. The Oracle policy knows $\wbar$ which is computed as the average between the $N=20$ considered user vectors. In Figure \ref{fig:TransferRegretLastFM} (and Figure \ref{fig:TransferRegretMovielens}) we displayed the cumulative regret suffered with respect to the optimal policy, which during each task $j \in [N]$ knows the true user parameter $\w_j$. The vertical yellow lines indicate the end of each task. From the presented results we can observe that both the proposed policies AVG-OFUL and RR-OFUL outperform the ITL approach, while the Oracle policy is consistent with Corollary $\ref{Cor:RightHTransfer}$ and Assumption \ref{Ass:LowVariance}.
\begin{figure}[t]
    \centering
    \includegraphics[width=0.45\textwidth]{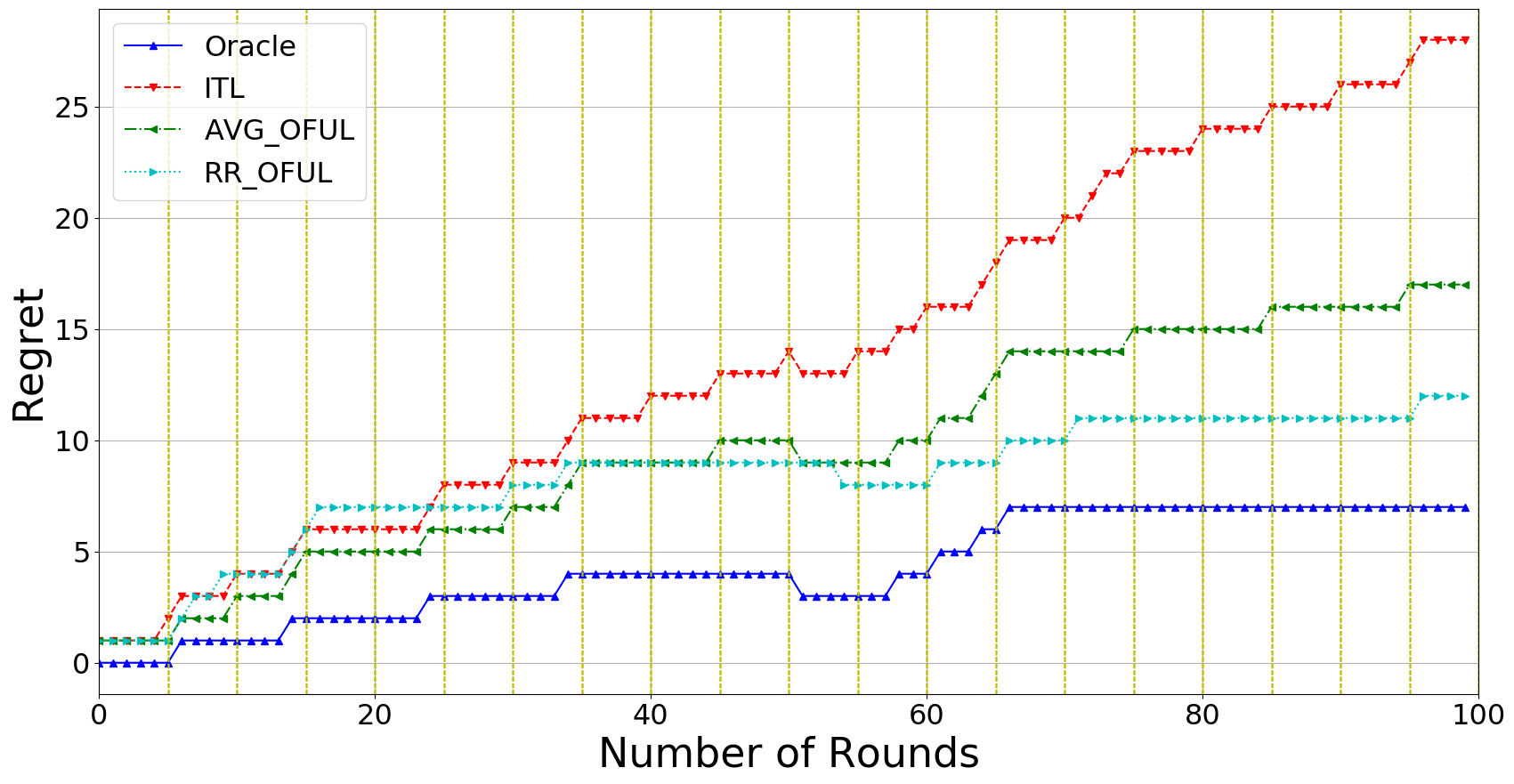}
    \caption{Empirical Transfer regret associated with Lastfm.}
    \label{fig:TransferRegretLastFM}
\end{figure}
\vspace{-.2truecm}
\paragraph{Movielens}
Here we consider the Movielens data \cite{Movielens}. It contains 1M anonymous ratings of approximately 3900 movies made by 6040 users. As before we first removed from the set of movies those with less than $500$ ratings, and from the set of users those with less than $200$ rated movies. This preprocessing procedure yields an user rating matrix of size 847 x 618. Unlike the Last.fm case, here adopting SVD to generate the arm/user vectors seems not appropriate. Indeed, by exploring the retrieved singular values, we could not find a subspace which provides a good approximation of the real ratings unless we keep all the latent features. Therefore, in order to find a set of similar users we observe better results by using the KMeans clustering algorithm over the user vectors. The results displayed in Figure \ref{fig:TransferRegretMovielens} were generated by running KMeans with $C=20$ clusters with user vectors of size $d=10$. We then picked all the resulting clusters by filtering out the clusterings with a silhouette value lower than $0.15$ and for each cluster of the clustering we have discarded those with less than $20$ users. Furthermore, in order to let the tasks be simpler, we reduced the variance of the noisy components affecting rewards to $0.1$. The difficulty in finding a valid set of similar tasks yields a high task misalignment, which is confirmed by the fact that the best performance occur for small value of $\lambda$. Indeed, Figure \ref{fig:TransferRegretMovielens} considers $\lambda=1$. Here the AVG-OFUL policy behaves almost equally to the ITL approach, conversely, the task misalignment caused bad performances to the RR-OFUL policy, confirming its higher sensitivity to task dissimilarity (see Theorem \ref{Th:TransferRegretBoundB}).
\begin{figure}[t]
    \centering
    \includegraphics[width=0.45\textwidth]{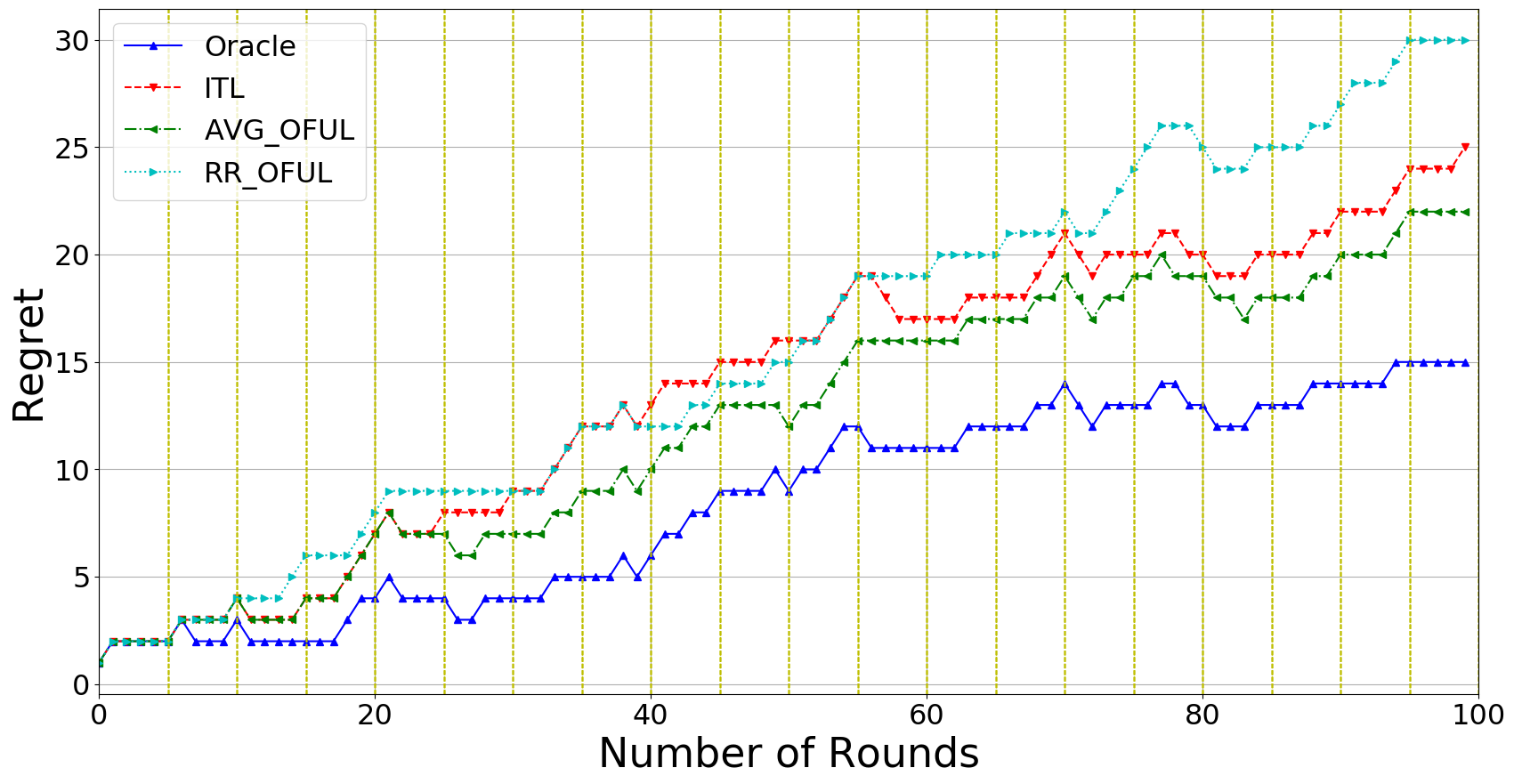}
    \caption{Empirical Transfer regret associated with Movielens.}
    \label{fig:TransferRegretMovielens}
\end{figure}
\section{Conclusions and Future Work}
In this work we studied a  meta-learning framework with stochastic linear bandit tasks. We have first introduced a novel regularized version of OFUL, where the regularization depends on the Euclidean distance to a bias vector. We showed that setting appropriately the bias leads a substantial improvement compared to learning each task in isolation. This observation motivated two alternative approaches to estimate this bias: while the first one may suffer a potentially high variance, the second might incur a strong bias.

In the future, it would be valuable to investigate the existence of unbiased estimators which do not suffer any variance. Furthermore, while in our analysis we set $\lambda=1/T\Var_\mathbf{h}$, in the future it would be also interesting to learn its value as part of the learning problem. Experimentally, we observed that when Assumption \ref{Ass:LowVariance} is satisfied, adopting the unbiased estimator yields better results than the second one, which is biased. One more direction of future research would be to extend other meta-learning approaches, such as those based on feature sharing, to the banding setting. Finally, a problem which remains to be studied is the combination of meta-learning with non-stochastic bandits.

\bibliographystyle{plainnat}
\bibliography{reference}

\begin{thebibliography}{40}
\providecommand{\natexlab}[1]{#1}
\providecommand{\url}[1]{\texttt{#1}}
\expandafter\ifx\csname urlstyle\endcsname\relax
  \providecommand{\doi}[1]{doi: #1}\else
  \providecommand{\doi}{doi: \begingroup \urlstyle{rm}\Url}\fi

\bibitem[Abbasi-Yadkori et~al.(2011)Abbasi-Yadkori, P\'{a}l, and
  Szepesv\'{a}ri]{OFUL}
Yasin Abbasi-Yadkori, D\'{a}vid P\'{a}l, and Csaba Szepesv\'{a}ri.
\newblock Improved algorithms for linear stochastic bandits.
\newblock In \emph{Proceedings of the 24th International Conference on Neural
  Information Processing Systems}, NIPS'11, pages 2312--2320, USA, 2011. Curran
  Associates Inc.
\newblock ISBN 978-1-61839-599-3.
\newblock URL \url{http://dl.acm.org/citation.cfm?id=2986459.2986717}.

\bibitem[Alquier et~al.(2017)Alquier, Mai, and Pontil]{LTLRegret}
Pierre Alquier, The~Tien Mai, and Massimiliano Pontil.
\newblock {Regret Bounds for Lifelong Learning}.
\newblock In Aarti Singh and Jerry Zhu, editors, \emph{Proceedings of the 20th
  International Conference on rtificial Intelligence and Statistics}, volume~54
  of \emph{Proceedings of Machine Learning Research}, pages 261--269, Fort
  Lauderdale, FL, USA, 20--22 Apr 2017. PMLR.
\newblock URL \url{http://proceedings.mlr.press/v54/alquier17a.html}.

\bibitem[Ando and Zhang(2005)]{MTLBatch}
Rie~Kubota Ando and Tong Zhang.
\newblock A framework for learning predictive structures from multiple tasks
  and unlabeled data.
\newblock \emph{J. Mach. Learn. Res.}, 6:\penalty0 1817--1853, December 2005.
\newblock ISSN 1532-4435.
\newblock URL \url{http://dl.acm.org/citation.cfm?id=1046920.1194905}.

\bibitem[Auer(2003)]{LinRel}
Peter Auer.
\newblock Using confidence bounds for exploitation-exploration trade-offs.
\newblock \emph{J. Mach. Learn. Res.}, 3:\penalty0 397--422, March 2003.
\newblock ISSN 1532-4435.
\newblock URL \url{http://dl.acm.org/citation.cfm?id=944919.944941}.

\bibitem[Auer et~al.(2002)Auer, Cesa-Bianchi, and Fischer]{UCB1}
Peter Auer, Nicol\`{o} Cesa-Bianchi, and Paul Fischer.
\newblock Finite-time analysis of the multiarmed bandit problem.
\newblock \emph{Mach. Learn.}, 47\penalty0 (2-3):\penalty0 235--256, May 2002.
\newblock ISSN 0885-6125.
\newblock \doi{10.1023/A:1013689704352}.
\newblock URL \url{https://doi.org/10.1023/A:1013689704352}.

\bibitem[Awerbuch and Kleinberg(2008)]{awerbuch2008online}
Baruch Awerbuch and Robert Kleinberg.
\newblock Online linear optimization and adaptive routing.
\newblock \emph{Journal of Computer and System Sciences}, 74\penalty0
  (1):\penalty0 97--114, 2008.

\bibitem[Azar et~al.(2013)Azar, Lazaric, and Brunskill]{SeqTLMAB}
Mohammad~Gheshlaghi Azar, Alessandro Lazaric, and Emma Brunskill.
\newblock Sequential transfer in multi-armed bandit with finite set of models.
\newblock In \emph{Proceedings of the 26th International Conference on Neural
  Information Processing Systems - Volume 2}, NIPS'13, pages 2220--2228, USA,
  2013. Curran Associates Inc.
\newblock URL \url{http://dl.acm.org/citation.cfm?id=2999792.2999860}.

\bibitem[Balcan et~al.(2019)Balcan, Khodak, and Talwalkar]{balcan2019provable}
Maria-Florina Balcan, Mikhail Khodak, and Ameet Talwalkar.
\newblock Provable guarantees for gradient-based meta-learning.
\newblock In \emph{International Conference on Machine Learning}, pages
  424--433, 2019.

\bibitem[Baxter(2000)]{LTLBaxter}
Jonathan Baxter.
\newblock A model of inductive bias learning.
\newblock \emph{J. Artif. Int. Res.}, 12\penalty0 (1):\penalty0 149--198, March
  2000.
\newblock ISSN 1076-9757.
\newblock URL \url{http://dl.acm.org/citation.cfm?id=1622248.1622254}.

\bibitem[Bogers(2010)]{bogers2010movie}
Toine Bogers.
\newblock Movie recommendation using random walks over the contextual graph.
\newblock In \emph{Proc. of the 2nd Intl. Workshop on Context-Aware Recommender
  Systems}, 2010.

\bibitem[Bubeck et~al.(2012)Bubeck, Cesa-Bianchi, et~al.]{RegretMAB}
S{\'e}bastien Bubeck, Nicolo Cesa-Bianchi, et~al.
\newblock Regret analysis of stochastic and nonstochastic multi-armed bandit
  problems.
\newblock \emph{Foundations and Trends{\textregistered} in Machine Learning},
  5\penalty0 (1):\penalty0 1--122, 2012.

\bibitem[Calandriello et~al.(2014)Calandriello, Lazaric, and
  Restelli]{SparseTLRL}
Daniele Calandriello, Alessandro Lazaric, and Marcello Restelli.
\newblock {Sparse Multi-task Reinforcement Learning}.
\newblock In \emph{{NIPS - Advances in Neural Information Processing Systems
  26}}, Montreal, Canada, December 2014.
\newblock URL \url{https://hal.inria.fr/hal-01073513}.

\bibitem[Cantador et~al.(2011)Cantador, Brusilovsky, and
  Kuflik]{Cantador:RecSys2011}
Iv\'{a}n Cantador, Peter Brusilovsky, and Tsvi Kuflik.
\newblock 2nd international workshop on information heterogeneity and fusion in
  recommender systems (hetrec 2011).
\newblock In \emph{Proceedings of the 5th ACM conference on Recommender
  systems}, RecSys 2011, New York, NY, USA, Chicago, IL, USA, 2011. ACM.
\newblock URL \url{http://ir.ii.uam.es/hetrec2011/index.html}.

\bibitem[Cavallanti et~al.(2010)Cavallanti, Cesa-Bianchi, and
  Gentile]{MTLLinear}
Giovanni Cavallanti, Nicol\`{o} Cesa-Bianchi, and Claudio Gentile.
\newblock Linear algorithms for online multitask classification.
\newblock \emph{J. Mach. Learn. Res.}, 11:\penalty0 2901--2934, December 2010.
\newblock ISSN 1532-4435.
\newblock URL \url{http://dl.acm.org/citation.cfm?id=1756006.1953026}.

\bibitem[Cella and Cesa-Bianchi(2019)]{cella2019stochastic}
Leonardo Cella and Nicol{\`o} Cesa-Bianchi.
\newblock Stochastic bandits with delay-dependent payoffs.
\newblock \emph{arXiv preprint arXiv:1910.02757}, 2019.

\bibitem[Cesa-Bianchi(2016)]{MAB}
Nicol{\`o} Cesa-Bianchi.
\newblock \emph{Multi-armed Bandit Problem}, pages 1356--1359.
\newblock Springer New York, New York, NY, 2016.
\newblock ISBN 978-1-4939-2864-4.
\newblock \doi{10.1007/978-1-4939-2864-4_768}.
\newblock URL \url{https://doi.org/10.1007/978-1-4939-2864-4_768}.

\bibitem[Cesa-Bianchi et~al.(2013)Cesa-Bianchi, Gentile, and Zappella]{GOB}
Nicol\`{o} Cesa-Bianchi, Claudio Gentile, and Giovanni Zappella.
\newblock A gang of bandits.
\newblock In \emph{Proceedings of the 26th International Conference on Neural
  Information Processing Systems - Volume 1}, NIPS'13, pages 737--745, USA,
  2013. Curran Associates Inc.
\newblock URL \url{http://dl.acm.org/citation.cfm?id=2999611.2999694}.

\bibitem[Chu et~al.(2011)Chu, Li, Reyzin, and Schapire]{LinUCB}
Wei Chu, Lihong Li, Lev Reyzin, and Robert Schapire.
\newblock Contextual bandits with linear payoff functions.
\newblock In Geoffrey Gordon, David Dunson, and Miroslav Dudík, editors,
  \emph{Proceedings of the Fourteenth International Conference on Artificial
  Intelligence and Statistics}, volume~15 of \emph{Proceedings of Machine
  Learning Research}, pages 208--214, Fort Lauderdale, FL, USA, 11--13 Apr
  2011. PMLR.
\newblock URL \url{http://proceedings.mlr.press/v15/chu11a.html}.

\bibitem[Denevi et~al.(2018{\natexlab{a}})Denevi, Ciliberto, Stamos, and
  Pontil]{LTLAroundCommonMean}
Giulia Denevi, Carlo Ciliberto, Dimitris Stamos, and Massimiliano Pontil.
\newblock Learning to learn around a common mean.
\newblock In S.~Bengio, H.~Wallach, H.~Larochelle, K.~Grauman, N.~Cesa-Bianchi,
  and R.~Garnett, editors, \emph{Advances in Neural Information Processing
  Systems 31}, pages 10169--10179. Curran Associates, Inc., 2018{\natexlab{a}}.
\newblock URL
  \url{http://papers.nips.cc/paper/8220-learning-to-learn-around-a-common-mean.pdf}.

\bibitem[Denevi et~al.(2018{\natexlab{b}})Denevi, Ciliberto, Stamos, and
  Pontil]{LTLStatGuarantees}
Giulia Denevi, Carlo Ciliberto, Dimitris Stamos, and Massimiliano Pontil.
\newblock Incremental learning-to-learn with statistical guarantees.
\newblock \emph{arXiv preprint arXiv:1803.08089}, 2018{\natexlab{b}}.

\bibitem[Denevi et~al.(2019)Denevi, Ciliberto, Grazzi, and Pontil]{LTLSGD}
Giulia Denevi, Carlo Ciliberto, Riccardo Grazzi, and Massimiliano Pontil.
\newblock Learning-to-learn stochastic gradient descent with biased
  regularization.
\newblock In Kamalika Chaudhuri and Ruslan Salakhutdinov, editors,
  \emph{Proceedings of the 36th International Conference on Machine Learning},
  volume~97 of \emph{Proceedings of Machine Learning Research}, pages
  1566--1575, Long Beach, California, USA, 09--15 Jun 2019. PMLR.
\newblock URL \url{http://proceedings.mlr.press/v97/denevi19a.html}.

\bibitem[Deshmukh et~al.(2017)Deshmukh, Dogan, and Scott]{MTLContMAB}
Aniket~An Deshmukh, Urun Dogan, and Clay Scott.
\newblock Multi-task learning for contextual bandits.
\newblock In \emph{Advances in Neural Information Processing Systems}, pages
  4848--4856, 2017.

\bibitem[Gentile et~al.(2014)Gentile, Li, and Zappella]{CLUB}
Claudio Gentile, Shuai Li, and Giovanni Zappella.
\newblock Online clustering of bandits.
\newblock In \emph{Proceedings of the 31st International Conference on
  International Conference on Machine Learning - Volume 32}, ICML'14, pages
  II--757--II--765. JMLR.org, 2014.
\newblock URL \url{http://dl.acm.org/citation.cfm?id=3044805.3044977}.

\bibitem[Gentile et~al.(2017)Gentile, Li, Kar, Karatzoglou, Zappella, and
  Etrue]{CAB}
Claudio Gentile, Shuai Li, Purushottam Kar, Alexandros Karatzoglou, Giovanni
  Zappella, and Evans Etrue.
\newblock On context-dependent clustering of bandits.
\newblock In Doina Precup and Yee~Whye Teh, editors, \emph{Proceedings of the
  34th International Conference on Machine Learning}, volume~70 of
  \emph{Proceedings of Machine Learning Research}, pages 1253--1262,
  International Convention Centre, Sydney, Australia, 06--11 Aug 2017. PMLR.
\newblock URL \url{http://proceedings.mlr.press/v70/gentile17a.html}.

\bibitem[Harper and Konstan(2015)]{Movielens}
F.~Maxwell Harper and Joseph~A. Konstan.
\newblock The movielens datasets: History and context.
\newblock \emph{ACM Trans. Interact. Intell. Syst.}, 5\penalty0 (4), December
  2015.
\newblock ISSN 2160-6455.
\newblock \doi{10.1145/2827872}.
\newblock URL \url{http://dx.doi.org/10.1145/2827872}.

\bibitem[Kuzborskij et~al.(2019)Kuzborskij, Cella, and Cesa-Bianchi]{SOFUL}
Ilja Kuzborskij, Leonardo Cella, and Nicol\`{o} Cesa-Bianchi.
\newblock Efficient linear bandits through matrix sketching.
\newblock In Kamalika Chaudhuri and Masashi Sugiyama, editors,
  \emph{Proceedings of Machine Learning Research}, volume~89 of
  \emph{Proceedings of Machine Learning Research}, pages 177--185. PMLR, 16--18
  Apr 2019.
\newblock URL \url{http://proceedings.mlr.press/v89/kuzborskij19a.html}.

\bibitem[Lai and Wei(1982)]{StochasticRidgeRegression}
Tze~Leung Lai and Ching~Zong Wei.
\newblock Least squares estimates in stochastic regression models with
  applications to identification and control of dynamic systems.
\newblock \emph{The Annals of Statistics}, 10:\penalty0 154--166, 1982.

\bibitem[Lattimore and Szepesv{\'a}ri(2020)]{lattimore2018bandit}
Tor Lattimore and Csaba Szepesv{\'a}ri.
\newblock \emph{Bandit algorithms}.
\newblock Cambridge University Press, 2020.

\bibitem[Li et~al.(2010)Li, Chu, Langford, and Schapire]{li2010contextual}
Lihong Li, Wei Chu, John Langford, and Robert~E Schapire.
\newblock A contextual-bandit approach to personalized news article
  recommendation.
\newblock In \emph{Proceedings of the 19th international conference on World
  wide web}, pages 661--670, 2010.

\bibitem[Liu et~al.(2018)Liu, Wei, Y., Yan, and Yang]{CrossDomain}
B.~Liu, Y.~Wei, Zhang Y., Z.~Yan, and Q.~Yang.
\newblock Transferable contextual bandit for cross-domain recommendation.
\newblock In \emph{In Thirty-Second AAAI Conference on Artificial
  Intelligence.}, 2018.

\bibitem[Maurer and Pontil(2013)]{MTLExcessRIsk}
Andreas Maurer and Massimiliano Pontil.
\newblock Excess risk bounds for multitask learning with trace norm
  regularization.
\newblock In \emph{Conference on Learning Theory}, pages 55--76, 2013.

\bibitem[Maurer et~al.(2013)Maurer, Pontil, and Romera-Paredes]{MTLSparse}
Andreas Maurer, Massimiliano Pontil, and Bernardino Romera-Paredes.
\newblock Sparse coding for multitask and transfer learning.
\newblock In \emph{Proceedings of the 30th International Conference on
  International Conference on Machine Learning - Volume 28}, ICML'13, pages
  II--343--II--351. JMLR.org, 2013.
\newblock URL \url{http://dl.acm.org/citation.cfm?id=3042817.3042932}.

\bibitem[Maurer et~al.(2016)Maurer, Pontil, and Romera-Paredes]{MTLBenefit}
Andreas Maurer, Massimiliano Pontil, and Bernardino Romera-Paredes.
\newblock The benefit of multitask representation learning.
\newblock \emph{J. Mach. Learn. Res.}, 17\penalty0 (1):\penalty0 2853--2884,
  January 2016.
\newblock ISSN 1532-4435.
\newblock URL \url{http://dl.acm.org/citation.cfm?id=2946645.3007034}.

\bibitem[Pentina and Urner(2016)]{LTLWeightedMajority}
Anastasia Pentina and Ruth Urner.
\newblock Lifelong learning with weighted majority votes.
\newblock In D.~D. Lee, M.~Sugiyama, U.~V. Luxburg, I.~Guyon, and R.~Garnett,
  editors, \emph{Advances in Neural Information Processing Systems 29}, pages
  3612--3620. Curran Associates, Inc., 2016.
\newblock URL
  \url{http://papers.nips.cc/paper/6095-lifelong-learning-with-weighted-majority-votes.pdf}.

\bibitem[Robbins(1952)]{RobbinsOriginal}
Herbert Robbins.
\newblock Some aspects of the sequential design of experiments.
\newblock \emph{Bulletin of the American Mathematical Society}, 58\penalty0
  (5):\penalty0 527--535, 1952.

\bibitem[Siegmund(2003)]{RobbinsReviewed}
David Siegmund.
\newblock Herbert robbins and sequential analysis.
\newblock \emph{Annals of statistics}, pages 349--365, 2003.

\bibitem[Smale and Zhou(2007)]{smale2007learning}
Steve Smale and Ding-Xuan Zhou.
\newblock Learning theory estimates via integral operators and their
  approximations.
\newblock \emph{Constructive approximation}, 26\penalty0 (2):\penalty0
  153--172, 2007.

\bibitem[Soare et~al.(2014)Soare, Alsharif, Lazaric, and Pineau]{MTLinMAB}
Marta Soare, Ouais Alsharif, Alessandro Lazaric, and Joelle Pineau.
\newblock Multi-task linear bandits.
\newblock In \emph{NIPS'14 Workshop on Transfer and Multi-task Learning}, 2014.

\bibitem[Villar et~al.(2015)Villar, Bowden, and Wason]{villar2015multi}
Sof{\'\i}a~S Villar, Jack Bowden, and James Wason.
\newblock Multi-armed bandit models for the optimal design of clinical trials:
  benefits and challenges.
\newblock \emph{Statistical science: a review journal of the Institute of
  Mathematical Statistics}, 30\penalty0 (2):\penalty0 199, 2015.

\bibitem[Zhang and Bareinboim(2017)]{CausalTLMAB}
Junzhe Zhang and Elias Bareinboim.
\newblock Transfer learning in multi-armed bandits: A causal approach.
\newblock In \emph{Proceedings of the 26th International Joint Conference on
  Artificial Intelligence}, IJCAI'17, pages 1340--1346. AAAI Press, 2017.
\newblock ISBN 978-0-9992411-0-3.
\newblock URL \url{http://dl.acm.org/citation.cfm?id=3171642.3171832}.

\end{thebibliography}

\clearpage
\onecolumn
\begin{center}
\setcounter{section}{0}
\textbf{\Large Supplemental Material}
\end{center}

\appendix
\section{Proof of Theorem \ref{Th:BiasedConfidenceSet}}
    \begin{proof}
		Starting from the biased-regularized estimation of Equation \ref{Eq:BiasedParameterEstimate},
		\begin{align*}
			\widehat{\mathbf{w}}^\mathbf{h}_t &= \left(\mathbf{V}^\lambda_t\right)^{-1} \mathbf{X}^\top_t ( \mathbf{Y}_{t} - \mathbf{X}_t \mathbf{h}) + \mathbf{h} = \left(\mathbf{V}^\lambda_t\right)^{-1} \mathbf{X}_t^\top ( \mathbf{X}_{t} \mathbf{w}^* + \boldsymbol{\eta}_t - \mathbf{X}_t\mathbf{h} )  + \mathbf{h}\\
		  &= \left(\mathbf{V}^\lambda_t\right)^{-1} \mathbf{X}_t^\top \mathbf{X}_{t} \mathbf{w}^* + \left(\mathbf{V}^\lambda_t\right)^{-1} \mathbf{X}_t^\top \boldsymbol{\eta}_t - \left(\mathbf{V}^\lambda_t\right)^{-1} \mathbf{X}_t^\top \mathbf{X}_t\mathbf{h} + \mathbf{h}
		\end{align*}
		Given this construction we can obtain the following equalities:
		\begin{align*}
			\widehat{\mathbf{w}}^\mathbf{h}_t - \mathbf{w}^* &= \left(\mathbf{V}^\lambda_t\right)^{-1} \mathbf{X}_t^\top \boldsymbol{\eta}_t + \mathbf{h} - \left(\mathbf{V}^\lambda_t\right)^{-1} \mathbf{X}_t^\top \mathbf{X}_t \mathbf{h} - \lambda \left(\mathbf{V}^\lambda_t\right)^{-1} \mathbf{w}^*\\
			&= \left(\mathbf{V}^\lambda_t\right)^{-1} \mathbf{X}_t^\top \boldsymbol{\eta}_t + \big( \lambda \left(\mathbf{V}^\lambda_t\right)^{-1} \big) \big( \mathbf{h} - \mathbf{w}^* \big)
		\end{align*}
		Then, for any $\mathbf{x}\in\mathbb{R}^d$ the following holds:
		\begin{align*}
			\mathbf{x}^\top \big( \widehat{\mathbf{w}}^\mathbf{h}_t - \mathbf{w}^* \big) &= \langle \mathbf{x} , \mathbf{X}_t^\top \boldsymbol{\eta}_t  \rangle_{\left(\mathbf{V}^\lambda_t\right)^{-1}} + \lambda \langle \mathbf{x} , \mathbf{h} \rangle_{\left(\mathbf{V}^\lambda_t\right)^{-1}} - \lambda \langle \mathbf{x} , \mathbf{w}^* \rangle_{\left(\mathbf{V}^\lambda_t\right)^{-1}} \\
			&\leq \|\mathbf{x}\|_{\left(\mathbf{V}^\lambda_t\right)^{-1}} \bigg(\|\mathbf{X}_t^\top \boldsymbol{\eta}_t\|_{\left(\mathbf{V}^\lambda_t\right)^{-1}} + \lambda \|\mathbf{h} - \mathbf{w}^*\|_{\left(\mathbf{V}^\lambda_t\right)^{-1}} \bigg)
		\end{align*}
		where in the last step we have applied Cauchy-Schwarz inequality. Plugging in $\mathbf{x} = \mathbf{V}^\lambda_t (\widehat{\mathbf{w}}^\mathbf{h}_t - \mathbf{w}^*)$ we obtain:
		\begin{equation}
			\|\widehat{\mathbf{w}}^\mathbf{h}_t - \mathbf{w}^* \|^2_{\mathbf{V}^\lambda_t} \leq \|\widehat{\mathbf{w}}^\mathbf{h}_t - \mathbf{w}^* \|_{\mathbf{V}^\lambda_t} \bigg(\|\mathbf{X}_t^\top \boldsymbol{\eta}_t\|_{\left(\mathbf{V}^\lambda_t\right)^{-1}} + \lambda \|\mathbf{h} - \mathbf{w}^*\|_{\left(\mathbf{V}^\lambda_t\right)^{-1}} \bigg) \nonumber
		\end{equation}
		finally by dividing both sides by $\norm{\widehat{\mathbf{w}}^\mathbf{h}_t - \mathbf{w}^* }_{\mathbf{V}^\lambda_t}$ we obtain:
		\begin{equation}
			\norm{\widehat{\mathbf{w}}^\mathbf{h}_t - \mathbf{w}^* }_{\mathbf{V}^\lambda_t} \leq \norm{\mathbf{X}_t^\top \boldsymbol{\eta}_t}_{\left(\mathbf{V}^\lambda_t\right)^{-1}} + \lambda \norm{\mathbf{h} - \mathbf{w}^*}_{\left(\mathbf{V}^\lambda_t\right)^{-1}}. \nonumber
		\end{equation}
		Finally we bound the noisy term $\norm{\mathbf{X}_t^\top \boldsymbol{\eta}_t}_{\left(\mathbf{V}^\lambda_t\right)^{-1}}$ by leveraging on Theorem 1 of \citep{OFUL}, obtaining:
		\begin{equation} \label{Eq:RegConfBound}
			\norm{\widehat{\mathbf{w}}^\mathbf{h}_t - \mathbf{w}^* }_{\mathbf{V}^\lambda_t} \leq R \sqrt{2  \log\bigg( \frac{\det\left(\mathbf{V}^\lambda_t\right)^{1/2}}{\det(\lambda I)^{1/2}\delta} \bigg)} + \lambda^{\frac{1}{2}}\norm{\mathbf{h} - \mathbf{w}^*}_2 =  \beta^\mathbf{h}_t(\delta)
		\end{equation}
		where we have used the fact that: $\norm{\mathbf{h} - \mathbf{w}^*}^2_{\left(\mathbf{V}^\lambda_t\right)^{-1}} \leq \frac{1}{\lambda_{\min}\left(\mathbf{V}^\lambda_t\right)} \norm{\mathbf{h} - \mathbf{w}^*}^2_2 \leq \frac{1}{\lambda}\norm{\mathbf{h} - \mathbf{w}^*}^2_{2}$.
	\end{proof}
	
\section{Proof of Lemma \ref{Lemma:RegretFixedBias}}
\begin{proof}
	We start by analysing the instantaneous regret as follows:
	\begin{align*}
		r_t &= \langle \mathbf{w}^*, \mathbf{x}^*_t\rangle - \langle \mathbf{w}^*, \mathbf{x}^\mathbf{h}_t \rangle = \langle \mathbf{w}^*, \mathbf{x}^*_t\rangle - \langle \widetilde{\mathbf{w}}^\mathbf{h}_t, \mathbf{x}^\mathbf{h}_t \rangle + \langle \widetilde{\mathbf{w}}^\mathbf{h}_t, \mathbf{x}^\mathbf{h}_t \rangle - \langle \mathbf{w}^*, \mathbf{x}^\mathbf{h}_t \rangle \\
		&\leq \langle \widetilde{\mathbf{w}}^\mathbf{h}_t, \mathbf{x}_t\rangle - \langle \mathbf{w}^*, \mathbf{x}^\mathbf{h}_t \rangle = \langle \widehat{\mathbf{w}}^\mathbf{h}_{t-1} - \mathbf{w}^*, \mathbf{x}^\mathbf{h}_t\rangle + \langle \widetilde{\mathbf{w}}^\mathbf{h}_t -  \widehat{\mathbf{w}}^\mathbf{h}_{t-1} , \mathbf{x}^\mathbf{h}_t\rangle \\
		&\leq \norm{ \widehat{\mathbf{w}}^\mathbf{h}_{t-1} - \mathbf{w}^* }_{\mathbf{V}^\lambda_{t-1}} \norm{\mathbf{x}^\mathbf{h}_t}_{\mathbf{V}^\lambda_{t-1}} + \norm{ \widetilde{\mathbf{w}}^\mathbf{h}_t -  \widehat{\mathbf{w}}^\mathbf{h}_{t-1}}_{\mathbf{V}^\lambda_{t-1}} \norm{\mathbf{x}^\mathbf{h}_t}_{\mathbf{V}^\lambda_{t-1}} \leq 2 \beta_{t-1}^{\mathbf{h}}(\delta)\norm{\mathbf{x}^\mathbf{h}_t}_{\mathbf{V}^\lambda_{t-1}}
	\end{align*}
	where in the first inequality we have leveraged on the fact that $\big( \widetilde{\mathbf{w}}^\mathbf{h}_t, \mathbf{x}_t \big)$ is optimistic and in the last the ellipsoid bound specified in Equation \ref{Eq:RegConfBound}. The bound of the cumulative regret follows from the bound of \citep{OFUL}, hence with probability at least $1 - \delta$, for all $T \geq 0$:
	\begin{align*}
		R(T, \mathbf{w}^*) &\leq \sqrt{T \sum_{t=1}^{T} {r_t}^2} \leq 4  \sqrt{ T \log\big(\det\left(\mathbf{V}^\lambda_t\right)\big)-\log\big(\det(\lambda \mathbf{I})\big)}\beta_T^\mathbf{h}(\delta)\\
		&\leq4 \sqrt{Td\log\bigg(1 + \frac{TL}{ \lambda d}\bigg)}\bigg(\lambda^\frac{1}{2}\norm{\mathbf{w}^* - \mathbf{h}}_2 + R \sqrt{2\log(1/\delta) + d \log\big(1 + TL/(\lambda d)\big)}\bigg)
	\end{align*}
	where the last two steps follow from Lemma 11 of \citep{OFUL} and the definition of $\beta^\mathbf{h}(\delta)$ (Equation \ref{Eq:RegConfBound}). The stated result is derived analogously to Corollary 19.3 of \cite{lattimore2018bandit} considering $\delta =  \frac{1}{T}$.
\end{proof}

\section{Proof of Corollary \ref{Cor:RightH}}
\begin{proof}
We start by considering the \textbf{oracle} scenario which is given by $\mathbf{h}=\mathbf{w}^*$.
\begin{align*}
    &\lim_{\lambda\to\infty} \Bigg[C \sqrt{Td\log\bigg(1 + \frac{TL}{ \lambda d}\bigg)}\bigg( R \sqrt{d \log(T + T^2 L/(\lambda d))}\bigg)\Bigg]\\
    &=C \sqrt{Td\log(1)}\bigg( R \sqrt{d \log(T + T^2 L/(\lambda d))}\bigg) = 0
\end{align*}
As far as the \textbf{independent task learning} scenario concerns, the following holds:
\begin{align*}
    &\lim_{\lambda\to\infty}C\sqrt{Td\log\bigg(1 + \frac{TL}{ \lambda d}\bigg)}\bigg(\lambda^\frac{1}{2}S + R \sqrt{d \log(T + T^2 L/(\lambda d))}\bigg)\\
    &=\lim_{\epsilon\to 0}C\sqrt{Td\log\big(1+\epsilon\big)}\Bigg(S\sqrt{\frac{TL}{\epsilon d}} + R \sqrt{d \log(T + T^2 L/(\lambda d))}\Bigg)\\
    &=\lim_{\epsilon\to 0} C\Bigg[ST\sqrt{\frac{Ld}{d}\frac{ \log\big(1+\epsilon\big)}{\epsilon}} + Rd\sqrt{T\log\big(1+\epsilon\big)\log(T + T^2 L/(\lambda d))}\Bigg]\\
    &=\lim_{\epsilon\to 0} C\Bigg[ST\sqrt{L} + Rd\sqrt{T\log\big(1+\epsilon\big) \log(T + T^2 L/(\lambda d))}\bigg)\Bigg] =CTS\sqrt{L}
\end{align*}
where we have used the substitution $\epsilon = \frac{TL}{\lambda d}$ and the fact that $\lim_{\epsilon\to 0}\frac{\log(1+\epsilon)}{\epsilon}\to 1$.
\end{proof}

\section{Proof of Theorem \ref{Th:TransferRegretBoundB}}\label{Sec:AppendixSolutionB}
We start by presenting two Lemmas which are necessary to obtain the final bound. Firstly, we need to introduce an additional variable:
\begin{equation*}
    \Hbar_{N,t+1}' = \left(\Vtilde_{N,t} \right)^{-1} \Bigg(\sum_{j=1}^{N}\mathbf{V}_{j,T} \mathbf{w}_j + \V_{N+1,t} \w_{N+1} \Bigg)
\end{equation*}
We will then split the analysis by studying separately the \textit{estimation error} $\hhat^\lambda_{N,t+1} - \Hbar_{N,t+1}'$ (Lemma \ref{Lemma:EstimationError}) and the \textit{estimation bias} $\Hbar_{N,t+1}' - \Hbar_{N,t+1}$ (Lemma \ref{Lemma:Bias}).
\begin{lemma}\label{Lemma:EstimationError}
    The following rewriting holds:
    \begin{equation*}
        \hhat^\lambda_{N,t+1} - \Hbar_{N,t+1}' = \left(\Vtilde^\lambda_{N,t}\right)^{-1} \Bigg( \sum_{j=1}^{N}\sum_{s=1}^{T} \mathbf{x}_{j,s} \eta_{j,s} + \sum_{s=1}^t \x_{N+1,s} \eta_{N+1,s} \Bigg) - \lambda \left(\Vtilde^\lambda_{N,t}\right)^{-1}\Hbar_{N,t+1}'
    \end{equation*}
\end{lemma}
\begin{proof}
    \begin{align*}
        \hhat^\lambda_{N,t+1} &= \left(\Vtilde^\lambda_{N,t}\right)^{-1}\btilde_{N,t} = \left(\Vtilde^\lambda_{N,t}\right)^{-1}\Bigg(\sum_{j=1}^{N}\sum_{s=1}^{T} \mathbf{x}_{j,s} y_{j,s} + \sum_{s=1}^{t} \x_{N+1,s} y_{N+1,s} \Bigg) \\
        &= \left(\Vtilde^\lambda_{N,t}\right)^{-1} \Bigg( \sum_{j=1}^{N}\sum_{s=1}^{T} \mathbf{x}_{j,s} \left(\mathbf{x}_{j,s}^\top \mathbf{w}_j + \eta_{j,s} \right) + \sum_{s=1}^{t} \x_{N+1,s} \left(\x_{N+1,s}^\top \w_{N+1} + \eta_{N+1,s} \right)\Bigg)\\
        &= \left(\Vtilde^\lambda_{N,t}\right)^{-1} \Bigg( \sum_{j=1}^{N}\sum_{s=1}^{T} \mathbf{x}_{j,s} \eta_{j,s} + \\
        &\quad+ \sum_{s=1}^t \x_{N+1,s} \eta_{N+1,s} \Bigg) + \left(\Vtilde^\lambda_{N,t}\right)^{-1} \Bigg( \sum_{j=1}^{N}\sum_{s=1}^{T} \mathbf{x}_{j,s} \mathbf{x}_{j,s}^\top \mathbf{w}_j + \sum_{s=1}^t \x_s \x_{N+1,s} ^\top \w_{N+1} \Bigg)\\
        &= \left(\Vtilde^\lambda_{N,t}\right)^{-1} \Bigg( \sum_{j=1}^{N}\sum_{s=1}^{T} \mathbf{x}_{j,s} \eta_{j,s} + \sum_{s=1}^t \x_{N+1,s} \eta_{N+1,s} \Bigg) + \\
        &\quad+ \left(\Vtilde^\lambda_{N,t}\right)^{-1} \Vtilde_{N,t} \left(\Vtilde_{N,t} \right)^{-1} \Bigg(\sum_{j=1}^{N} \mathbf{V}_{j,T} \mathbf{w}_j + \V_{N+1,t} \w_{N+1} \Bigg)\\
        &= \left(\Vtilde^\lambda_{N,t}\right)^{-1} \Bigg( \sum_{j=1}^{N}\sum_{s=1}^{T} \mathbf{x}_{j,s} \eta_{j,s} + \sum_{s=1}^t \x_{N+1,s} \eta_{N+1,s} \Bigg) + \left(\Vtilde^\lambda_{N,t}\right)^{-1} \Vtilde_{N,t} \Hbar_{N,t+1}' + \\
        &\quad+ \lambda \left(\Vtilde^\lambda_{N,t}\right)^{-1} \Big[ \Hbar_{N,t+1}' - \Hbar_{N,t+1}' \Big]\\
        &= \left(\Vtilde^\lambda_{N,t}\right)^{-1} \Bigg( \sum_{j=1}^{N}\sum_{s=1}^{T} \mathbf{x}_{j,s} \eta_{j,s} + \sum_{s=1}^t \x_{N+1,s} \eta_{N+1,s} \Bigg) + \Hbar_{N,t+1}' - \lambda \left(\Vtilde^\lambda_{N,t}\right)^{-1}\Hbar_{N,t+1}'
    \end{align*}
    which gives the claimed result.
\end{proof}

\begin{lemma}\label{Lemma:Bias}
According to what we have done in Section \ref{Sec:SolutionA}, we use:
\[
    \Hbar_{N,t+1} = \frac{1}{NT+t} \left( \sum_{j=1}^{N} T \mathbf{w}_j + t \w_{N+1} \right).
\]
Differently from $\Hbar_{N,t}'$ this definition is a weighted average of the vectors of the $N$ completed tasks. Hence, we have:
\begin{align*}
    \norm{\wbar - \Hbar_{N,t}'}
    &\leq \frac{1}{NT+t} \sum_{j=1}^{N}\bigg[ \norm{\wbar - \Hbar_{N,t}} + (NT+t) \norm{\Hbar_{N,t} - \Hbar_{N,t}}'\bigg]\\
    &= H_\rho(N + 1,\wbar) + \norm{\Hbar_{N,t} - \Hbar_{N,t}'} 
\end{align*}
where we have denoted with $H_\rho(N + 1,\wbar)$ according to what we have done in Section \ref{Sec:SolutionA}. We can now focus on the term $\norm{\Hbar_{N,t}' - \Hbar_{N,t}}$ which can be equivalently rewritten as follows:
\begin{align*}
    \norm{\Hbar_{N,t+1}' - \Hbar_{N,t+1}} &= \norm{\left(\Vtilde_{N,t}\right)^{-1} \sum_{j-1}^{N} \left( \mathbf{V}_{j,T} \mathbf{w}_j + \V_{N+1,t} \w_{N+1} \right) - \Hbar_{N,t}}\\
    &\leq \sum_{j=1}^{N} \Big| \Vtilde_{N,t}^{-1} \mathbf{V}_{j,T} \Big| \norm{ \mathbf{w}_j - \Hbar_{N,t}} + \Big| \Vtilde_{N,t}^{-1} \V_{N+1,t} \Big| \norm{ \w_{N+1} - \Hbar_{N,t}}\\
    &\leq \sum_{j=1}^{N} H_\rho(N+1,\mathbf{w}_j) \Big| \Vtilde_{N,t}^{-1} \mathbf{V}_{j,T} \Big| + H_\rho(N+1,\w) \Big|\Vtilde_{N,t}^{-1} \V_t \Big|\\
    &= \sum_{j=1}^{N} H_\rho(N+1,\mathbf{w}_j) \sigma_{\max} \bigg(\mathbf{V}_{j,t}\Vtilde^{-1}_{N,t} \bigg) + H_\rho(N+1,\w_{N+1}) \sigma_{\max} \bigg(\mathbf{V}_{j,t}\Vtilde^{-1}_{N,t} \bigg)\\
    &\leq (N+1) \max_{j=1,\dots,N+1} \Bigg( H_\rho(N+1, \w_j) \sigma_{\max} \bigg( \mathbf{V}_{j,t}\Vtilde^{-1}_{N,t} \bigg) \Bigg) = (N+1) \max_{j=1,\dots,N+1} \widetilde{H}(N+1, \w_j)
\end{align*}
We have used the fact that the matrix norm of a given matrix $A$ induced by the Euclidean norm corresponds to the spectral norm, which is the largest singular value of the matrix $\sigma_{\max}(A)$ .
\end{lemma}

\subsection{Proof of Theorem \ref{Th:TransferRegretBoundB}}
We start the analysis from the result of Lemma \ref{Lemma:TransferRegretBound}:
\begin{equation*}
    \mathcal{R}(T,\rho) \leq d \sqrt{T \log\left(1 + \frac{T^2L \bigg(\mathbb{E}_{\mathbf{w}\sim\rho}\Big[\norm{\mathbf{w} - \mathbf{h}}^2_2\Big]\bigg)}{d}\right)}
\end{equation*}
we can then set the hyperparameter $\mathbf{h}=\hhat_{N,T}^\lambda$ and focusing on the first term in brackets we obtain:
\begin{equation*}
    \sqrt{\mathbb{E}_{\mathbf{w}\sim\rho}\Bigg[\norm{\mathbf{w} - \hhat_{N,T}^\lambda}^2_2\Bigg]} \leq \sqrt{\var_\wbar} + \sqrt{\epsilon_{N,t}(\rho)}
\end{equation*}
According to  Lemma \ref{Lemma:Bias} the following rewriting holds:
\begin{equation*}
    \sqrt{\epsilon_{N,t}(\rho)} \leq H_\rho(N + 1,\wbar) +  (N+1) \max_{j=1,\dots,N+1} \widetilde{H}(N+1, j) + \norm{\Hbar_{N,T}' - \hhat_{N,T}^\lambda}_2
\end{equation*}
It remains only to apply Lemma \ref{Lemma:EstimationError} which gives:
\begin{align*}
    \norm{\Hbar_{N,T}' - \hhat_{N,T}^\lambda}_2 &=  \norm{\left(\Vtilde^\lambda_{N,T}\right)^{-1} \left( \sum_{j=1}^{N}\sum_{s=1}^{T} \mathbf{x}_{j,s} \eta_{j,s} + \sum_{s=1}^T \x_s \eta_s \right)}_2+ \norm{\lambda \left(\Vtilde^\lambda_{N,T}\right)^{-1}\Hbar_{N,T}'}_2 \\
    &\leq \norm{\sum_{j=1}^{N}\sum_{s=1}^{T} \mathbf{x}_{j,s} \eta_{j,s} + \sum_{s=1}^T \x_s \eta_s}_{\left(\Vtilde^\lambda_{N,T}\right)^{-2}} + \lambda \norm{\Hbar_{N,T}'}_{_{\left(\Vtilde^\lambda_{N,T}\right)^{-2}}}\\
    &\leq \frac{1}{\lambda^\frac{1}{2}_{\min}(\Vtilde^\lambda_{N,T})}\norm{\sum_{j=1}^{N}\sum_{s=1}^{T} \mathbf{x}_{j,s} \eta_{j,s} + \sum_{s=1}^T \x_s \eta_s}_{\left(\Vtilde^\lambda_{N,T}\right)^{-1}} + \frac{1}{\lambda_{\min}(\Vtilde^\lambda_{N,T})} \norm{\Hbar_{N,T}'}_{2}\\
    &\leq \frac{1}{\lambda^\frac{1}{2}_{\min}(\Vtilde^\lambda_{N,T})} R\sqrt{2\log\bigg(T+\frac{(NT + T)TL^2}{\lambda d}\bigg)} + \norm{\Hbar_{N,T}' - \Hbar_{N,T}}_{2} + \frac{1}{\lambda_{\min}(\Vtilde^\lambda_{N,T})} \norm{\Hbar_{N,T}}_2\\
    &\leq \frac{1}{\lambda^\frac{1}{2}_{\min}(\Vtilde^\lambda_{N,T})} R\sqrt{2\log\bigg(T+\frac{(NT + T)TL^2}{\lambda d}\bigg)} + \norm{\Hbar_{N,T}' - \Hbar_{N,T}}_{2} + \frac{S}{\lambda_{\min}(\Vtilde^\lambda_{N,T})}\\
    &\leq \frac{1}{\lambda^\frac{1}{2}_{\min}(\Vtilde^\lambda_{N,T})} R\sqrt{2\log\bigg(T+\frac{(NT + T)TL^2}{\lambda d}\bigg)} + (N+1) \max_{j=1,\dots,N+1} \widetilde{H}(N+1, j) + \frac{S}{\lambda_{\min}(\Vtilde^\lambda_{N,T})}\\
\end{align*}
where in the last inequality we have applied once more Lemma \ref{Lemma:Bias}.
We can now introduce $\nu_{\min} = \lambda_{\min}\big(\Vtilde_{N,T}\big)$ as the minimum eigenvalue of the global covariance matrix without regularization which gives the following bound:
\begin{equation*}
    \frac{1}{\lambda_{\min}(\Vtilde^\lambda_{N,T})}\leq \frac{1}{\lambda + \nu_{\min}}
\end{equation*}
putting everything together gives the claimed result:
\begin{align*}
    \sqrt{\epsilon_{N,T}(\rho)} &\leq H_\rho(N + 1,\wbar) + 2 (N+1) \max_{j=1,\dots,N+1} \widetilde{H}(N+1,\w_j)  + \\
    &\quad+  \frac{1}{(\lambda + \nu_{\min})^{\frac{1}{2}}} R\sqrt{2\log\bigg(T+\frac{(NT + T)TL^2}{\lambda d}\bigg)} + \frac{S}{\lambda + \nu_{\min}}
\end{align*}

\end{document}